%% file: iclr2026_conference.tex
\documentclass{article} 
\usepackage{iclr2027_conference,times}

\input{math_commands.tex}

\usepackage{hyperref}
\usepackage{url}
\usepackage{amsmath,amssymb}
\usepackage{amsthm}
\usepackage{booktabs}
\usepackage{graphicx}
\usepackage{tikz}
\usetikzlibrary{arrows.meta,positioning}

\newtheorem{assumption}{Assumption}
\newtheorem{proposition}{Proposition}
\newtheorem{theorem}{Theorem}
\newtheorem{lemma}{Lemma}
\newtheorem{remark}{Remark}

\title{Intervention, Not Shared Latents: \\ Blocking Visual Shortcuts
in Audio--Video Generation}

\iclrfinalcopy
\author{Jian Xu\textsuperscript{1,2}, \quad Delu Zeng\textsuperscript{3}, \quad
John Paisley\textsuperscript{4}\\[0.5em]
\textsuperscript{1}RIKEN iTHEMS \quad \textsuperscript{2}RIKEN AIP \quad
\textsuperscript{3}South China University of Technology \quad
\textsuperscript{4}Columbia University\\[0.3em]
\texttt{jian.xu@riken.jp}}

\newcommand{\xv}{x_{\mathrm{v}}}
\newcommand{\xa}{x_{\mathrm{a}}}
\newcommand{\Ev}{\mathrm{e}}
\newcommand{\nuis}{\mathrm{n}}

\begin{document}

\maketitle
\lhead{}\chead{}\rhead{}  

\begin{abstract}
Joint audio--video (AV) generators are trained on data in which \emph{what an event
looks like} and \emph{what it sounds like} are spuriously correlated. We present a
\emph{controlled causal study} of the resulting failure mode. In an AV structural causal
model where the audio is, by construction, independent of the video's nuisance appearance,
models that let audio read video directly---through cross-attention or a shared
latent---learn a \emph{visual shortcut}: they predict sound from appearance rather than the
causal event and, when the appearance--event correlation is broken at test time, synthesize
the wrong event's sound. Crucially, the popular remedy of routing both modalities through a
\emph{shared common-cause latent} does \emph{not} fix this---a bottleneck, an unsupervised
shared/private factorization, and a faithful shared-prior model all grab the appearance proxy
and fail like the direct model. Blocking the shortcut instead requires an \emph{intervention
on the nuisance}: under the stated assumptions we prove that counterfactual invariance is
necessary and sufficient to identify the causal predictor, and we verify the mechanism from
feature-vector SCMs to procedural pixel video, real images with spectrogram audio, moving
real digits, and a conditional generator. On a \emph{real, pretrained} V2A generator
(MMAudio), an input-intervention test shows the model is far from invariant to
sound-irrelevant edits, though a generic-noise control reveals it is broadly input-brittle
rather than specifically colour-shortcutting---clean isolation of the shortcut needs the
controlled confounds our synthetic studies provide. We characterize \emph{when} the shortcut
occurs, compare the objective against supervised counterfactual augmentation, and isolate the
\emph{unknown-nuisance} regime---where the intervention cannot be applied---as the central
open problem.
\end{abstract}

\section{Introduction}

A generative model of paired audio and video must decide \emph{where the sound
comes from}. In natural data the honest answer is an \emph{event}: a cup hits
the floor, and that single cause produces both a visual impact and an acoustic
one. But the same datasets also contain strong \emph{nuisance} regularities---a
given instrument almost always looks a certain way, a given material almost
always has a certain texture---so that appearance and event are tightly, and
partly spuriously, coupled during training. A model is then free to take a
shortcut \citep{geirhos2020shortcut}: rather than reading the causal event
content out of the video, it can read the nuisance appearance, use it as a proxy
for the event, and generate the associated sound. In-distribution this is
indistinguishable from understanding; under any shift that decouples appearance
from event---a new material, an occluded source, a changed viewpoint---it is
wrong.

Recent joint AV diffusion models couple the two streams either by direct
cross-attention \citep{ruan2023mm,hacohen2026ltx} or, more recently, by a
\emph{shared} spatio-temporal prior that both branches align to
\citep{liu2025javisdit}. The latter is often motivated precisely as a way to avoid
spurious cross-modal coupling. Independent physical-commonsense benchmarks,
however, report that current video-to-audio and joint models remain sensitive
to exactly the visual factors they should ignore
\citep{cui2026joint,li2026benchmarking}, suggesting the shortcut survives these
architectural choices. This paper asks a sharp version of the question: \emph{is
sharing a latent enough to remove the visual shortcut, and if not, what is?}

Our answer is negative for architecture and positive for intervention. We first
formalize the shortcut as a confounded structural causal model (SCM) in which
audio is, by construction, independent of the video's nuisance appearance given
the event and scene. We then show, in controlled settings where ground truth is
known exactly, three things:
\begin{itemize}
  \item \textbf{The shortcut is real and catastrophic.} A direct model attains
  near-zero in-distribution audio error but fails by up to two orders of
  magnitude when appearance is decoupled from event, and the failure is
  \emph{slice-local}: it appears only on the event classes that were confounded
  in training.
  \item \textbf{Sharing a latent does not fix it.} A common-cause bottleneck,
  \emph{and} an unsupervised shared/private factorization that generates audio
  from the shared code alone, both fail as badly as the direct model: when
  appearance perfectly predicts the event, the ``shared'' code simply encodes
  appearance. This is consistent with known impossibility results for
  unsupervised disentanglement \citep{locatello2019challenging}.
  \item \textbf{Counterfactual invariance is necessary and sufficient under the
  stated SCM assumptions.} We enforce it through a counterfactual-consistency
  (do-operator) objective: forcing generated audio to be invariant to
  interventions on the nuisance restores near-oracle robustness at negligible
  in-distribution cost and without harming unconfounded data. When the nuisance
  can be augmented (e.g.\ recoloring, restyling), a single feasible
  augmentation-based intervention suffices---and, importantly, the mechanism
  survives the move from abstract feature vectors to rendered pixel video.
\end{itemize}

\emph{What is and is not new here} and the study's scope are discussed in Appendix~\ref{app:positioning}; limitations and the unknown-nuisance problem in Appendix~\ref{sec:limitations}.

\section{Related Work}

\paragraph{Joint audio--video generation.}
Joint AV diffusion models synchronize modalities through cross-modal attention
or shared backbones \citep{ruan2023mm,xing2024seeing,hacohen2026ltx},
and more recent work introduces an explicit shared spatio-temporal prior that
both branches align to \citep{liu2025javisdit}. These designs target
\emph{synchrony}; we instead target \emph{causal correctness} of the sound given
the event, and we show that a shared latent alone does not deliver it. Physical
reasoning benchmarks for AV generation \citep{cui2026joint,li2026benchmarking}
document that generated audio tracks spurious visual factors, providing external
evidence that the shortcut we isolate occurs in practice.

\paragraph{Shortcut learning, causality, and invariance.}
Shortcut learning \citep{geirhos2020shortcut} and its remedies---invariant risk
minimization \citep{arjovsky2019invariant}, counterfactual invariance
\citep{veitch2021counterfactual}, and neural structural causal models
\citep{yang2021causalvae}---provide the conceptual backbone for our objective.
We contribute the specific instantiation for \emph{generative} audio--video
models: audio generation invariant to interventions on visual nuisance,
together with controlled evidence that the intervention, not the shared-latent
architecture, is what removes the shortcut.

\section{The Visual Shortcut in Audio--Video Generation}
\label{sec:setup}

\paragraph{Structural causal model.}
Let an \emph{event} $\Ev$ and a \emph{scene} $s$ (e.g.\ source distance) be the
common causes of a video $\xv$ and an audio signal $\xa$, and let $\nuis$ be a
\emph{nuisance} appearance factor (material, texture, color) that affects the
video only:
\begin{equation}
  \xv = g_{\mathrm{v}}(\Ev, \nuis) + \epsilon_{\mathrm{v}}, \qquad
  \xa = g_{\mathrm{a}}(\Ev, s) + \epsilon_{\mathrm{a}}, \qquad
  \xa \perp \nuis \mid (\Ev, s).
\end{equation}
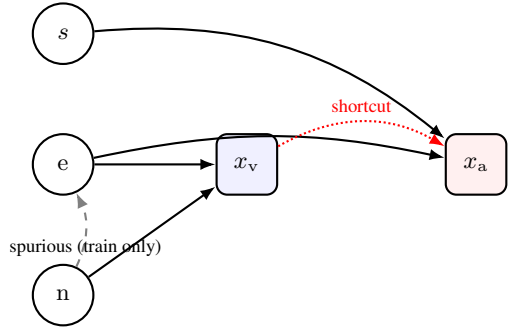
\begin{figure}[t]
\centering
\begin{tikzpicture}[
  node distance=9mm and 16mm,
  var/.style={circle,draw,minimum size=8mm,inner sep=1pt,font=\small},
  obs/.style={rectangle,draw,rounded corners,minimum size=8mm,font=\small},
  ->,>={Latex[length=2mm]},thick]
  \node[var] (e) {$\Ev$};
  \node[var,below=of e] (n) {$\nuis$};
  \node[var,above=of e] (s) {$s$};
  \node[obs,right=of e,fill=blue!6] (xv) {$\xv$};
  \node[obs,right=22mm of xv,fill=red!6] (xa) {$\xa$};
  \draw (e) -- (xv); \draw (n) -- (xv);
  \draw (e) to[bend left=12] (xa); \draw (s) to[bend left=22] (xa);
  \draw[dashed,gray] (n) to[bend right=25] node[midway,below,font=\scriptsize,black]{spurious (train only)} (e);
  \draw[red,densely dotted,->] (xv) to[bend left=30] node[midway,above,font=\scriptsize]{shortcut} (xa);
\end{tikzpicture}
\caption{The audio--video SCM. The event $\Ev$ (and scene $s$) cause both video
and audio; the nuisance appearance $\nuis$ affects only the video, and
$\xa\perp\nuis\mid(\Ev,s)$. Training confounds $\nuis$ with $\Ev$ (dashed), so a
model can generate audio via the causal path ($\Ev$ read out of $\xv$) or via the
\textcolor{red}{shortcut} (read $\nuis$ off $\xv$). Our objective forbids the
shortcut by enforcing invariance to interventions on $\nuis$.}
\label{fig:scm}
\end{figure}

By construction the sound depends on the event and scene but \emph{not} on the
video's appearance. Training data, however, is confounded: the nuisance is
correlated with the event, $\nuis = \Ev$ (in appropriate coordinates) with
probability $\rho_{\mathrm{spur}}$, and independent otherwise. A predictor or
conditional generator $f(\xv, s) \approx \xa$ may therefore route through either
the \emph{causal} path (event content in $\xv$) or the \emph{shortcut} (read
$\nuis$ off $\xv$, use it as a proxy for $\Ev$). The two are indistinguishable
in-distribution; we evaluate under the counterfactual shift $\mathrm{do}(\nuis
\!\perp\! \Ev)$ that decouples appearance from event.

\paragraph{Why a shared latent is not enough.}
A common-cause model infers a shared code $z_{\mathrm{s}} = E(\xv)$ and generates
audio from $z_{\mathrm{s}}$ alone. If $\nuis$ perfectly predicts $\Ev$ in
training, then $\nuis$ is a \emph{sufficient statistic} for the audio-relevant
information, and nothing in a reconstruction- or prediction-driven objective
prevents $z_{\mathrm{s}}$ from encoding appearance instead of event. A
shared/private split $\xv \!\mapsto\! (z_{\mathrm{s}}, z_{\mathrm{v}})$ with
audio from $z_{\mathrm{s}}$ does not resolve this either, in line with the
impossibility of unsupervised disentanglement without inductive bias or
supervision \citep{locatello2019challenging}. We confirm both failures
empirically in Section~\ref{sec:experiments}.

\section{Method: Counterfactual Consistency}
\label{sec:method}

We block the shortcut by requiring the generated audio to be invariant to
interventions on the nuisance. Let $\mathrm{do}(\nuis := \nuis')$ denote
replacing the appearance while holding the event and scene fixed, producing a
counterfactual video $\xv^{\,\prime}$. We add to the generation/prediction loss
$\mathcal{L}_{\mathrm{gen}} = \| f(\xv, s) - \xa \|^2$ the
\emph{counterfactual-consistency} term
\begin{equation}
  \mathcal{L}_{\mathrm{cf}} = \mathbb{E}_{\nuis' } \,
  \big\| f(\xv, s) - f\big(\xv^{\,\prime}, s\big) \big\|^2,
  \qquad
  \mathcal{L} = \mathcal{L}_{\mathrm{gen}} + \lambda\, \mathcal{L}_{\mathrm{cf}} .
\end{equation}
This is a do-operator objective: it enforces $f(\mathrm{do}(\nuis)) $
invariance, which under the SCM of Section~\ref{sec:setup} is exactly the
condition that $f$ use only the event/scene path.
Figure~\ref{fig:cfpairs} shows what the intervention looks like concretely in one
of our real-image settings: the event (the digit) is held fixed while the nuisance
(its colour) is resampled, and the objective demands that the generated audio not
change across the row.

\begin{figure}[t]
\centering
\includegraphics[width=0.72\linewidth]{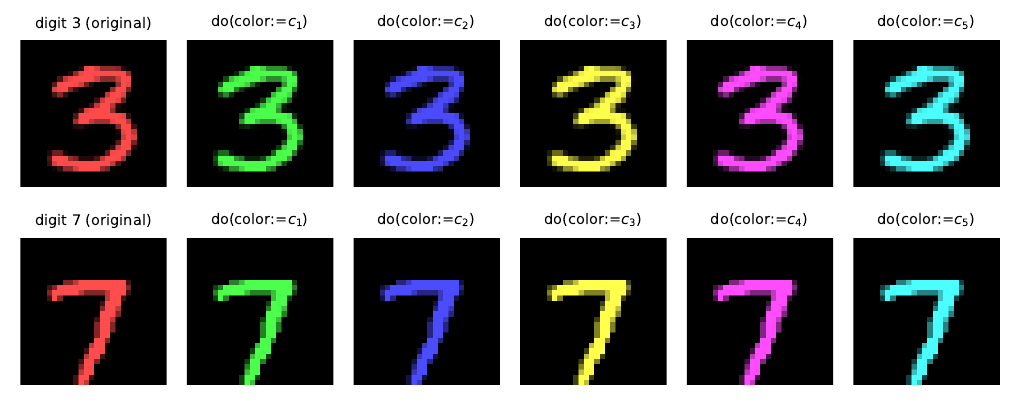}
\caption{The counterfactual intervention $\mathrm{do}(\nuis:=\nuis')$ made concrete
(Colored-MNIST): the event---the digit---is fixed, the nuisance---its colour---is
resampled. $\mathcal{L}_{\mathrm{cf}}$ penalizes any change in the generated audio
along such a row.}
\label{fig:cfpairs}
\end{figure} Two regimes arise. When the
nuisance is \emph{identifiable and augmentable} (color, texture, style), the
counterfactual $\xv^{\,\prime}$ is produced by an appearance augmentation and
$\mathcal{L}_{\mathrm{cf}}$ is directly computable---the practical recipe we
validate. When the nuisance is \emph{unknown}, no augmentation is available, and
one must fall back on representation learning; we show this fallback currently
fails, and we flag it as the open problem (Appendix~\ref{sec:limitations}).

\section{Theory}
\label{sec:theory}

We make the claims of Sections~\ref{sec:setup}--\ref{sec:method} precise. Let the
latent factors $(\Ev,s,\nuis)$ have joint law $P$, let $\xv=g_{\mathrm{v}}(\Ev,\nuis)$,
and write $\mu(\Ev,s):=\mathbb{E}[\xa\mid \Ev,s,\nuis]$. All proofs are deferred to
Appendix~\ref{app:proofs}.

\begin{assumption}[Causal audio]\label{as:a1}
$\mathbb{E}[\xa\mid \Ev,s,\nuis]=\mu(\Ev,s)$ does not depend on $\nuis$; equivalently
$\xa \perp \nuis \mid (\Ev,s)$.
\end{assumption}
\begin{assumption}[Recoverability]\label{as:a2}
$g_{\mathrm{v}}$ is injective in $(\Ev,\nuis)$, so any measurable predictor factors as
$f(\xv,s)=\tilde f(\Ev,\nuis,s)$.
\end{assumption}

Write $R_D(f)=\mathbb{E}_{D}\|f(\xv,s)-\xa\|^2$. Since
$R_D(f)=\mathbb{E}_D\|\tilde f-\mu\|^2+\mathrm{const}$, minimizing risk is equivalent to
matching $\mu$. The \emph{causal predictor} is $f^\star(\xv,s):=\mu(\Ev,s)$. Let
$P_{\mathrm{tr}}$ be the (confounded) training law and $P_{\mathrm{te}}$ the counterfactual
test law with $\nuis\perp(\Ev,s)$ but the same $(\Ev,s)$ marginal.

\begin{proposition}[Confounding makes ERM shortcut-blind]\label{prop:blind}
Suppose $\nuis=\pi(\Ev)$ $P_{\mathrm{tr}}$-a.s.\ for an injective $\pi$ (perfect confounding),
and define the shortcut predictor $f_{\mathrm{sc}}(\xv,s):=\mu(\pi^{-1}(\nuis),s)$. Then
(i) $R_{\mathrm{tr}}(f_{\mathrm{sc}})=R_{\mathrm{tr}}(f^\star)$, so both minimize the training
risk; and (ii)
$R_{\mathrm{te}}(f_{\mathrm{sc}})-R_{\mathrm{te}}(f^\star)
=\mathbb{E}\,\|\mu(\Ev',s)-\mu(\Ev,s)\|^2>0$
with $\Ev'=\pi^{-1}(\nuis)\perp \Ev$, whenever $\mu(\cdot,s)$ is non-constant.
\end{proposition}
Thus training risk does not identify $f^\star$: an ERM learner may return the shortcut, whose
OOD excess risk is exactly the event-explained audio variance. This is the phase transition of
Table~\ref{tab:onset} (any decorrelated mass breaks $\nuis=\pi(\Ev)$ and restores identifiability).

\begin{proposition}[Shared-latent / bottleneck insufficiency]\label{prop:bottleneck}
Consider generators that produce audio from a code $z=E(\xv)$ chosen to minimize
$R_{\mathrm{tr}}$ under an information (or dimension) constraint, and assume (as used in the
proof) that the event is audio-identifiable on a set of scenes of positive measure---formally
$\forall\,\Ev\neq\Ev'$, $P_s[\mu(\Ev,s)\neq\mu(\Ev',s)]>0$---so that an
audio-optimal code must determine the event. Under the hypothesis of
Proposition~\ref{prop:blind}, $E(\xv)=\nuis$ attains the optimal training risk at rate
$I(z;\xv)=H(\nuis)=H(\Ev)$, i.e.\ it is simultaneously risk-optimal and maximally compressive;
hence the bottleneck objective admits $E(\xv)=\nuis$ among its minimizers: it \emph{cannot guarantee} the event is preferred over the nuisance. This does not mean every shared-latent model must fail---only that sharing a latent is not \emph{sufficient} to rule the shortcut out. The same holds for a shared/private split with audio decoded from the shared code.
\end{proposition}
This matches the collapse of the bottleneck and unsupervised shared/private models in
Table~\ref{tab:pixels}, and is consistent with the impossibility of unsupervised disentanglement
without inductive bias \citep{locatello2019challenging}.

\begin{theorem}[Counterfactual invariance identifies the causal predictor, under
known and faithful nuisance interventions (Assumptions~\ref{as:a1}--\ref{as:a2})]\label{thm:cf}
Let $\mathcal{F}_{\mathrm{inv}}=\{f:\ \tilde f(\Ev,\nuis,s)\text{ is independent of }\nuis\}$.
Then:
\begin{enumerate}
\item[(i)] \textbf{Identification.} $\arg\min_{f\in\mathcal{F}_{\mathrm{inv}}}R_{\mathrm{tr}}(f)=f^\star$,
and this predictor also minimizes $R_{\mathrm{te}}$ over all measurable $f$.
\item[(ii)] \textbf{Gap closure.} For every $f\in\mathcal{F}_{\mathrm{inv}}$,
$R_{\mathrm{tr}}(f)=R_{\mathrm{te}}(f)$ (the two laws share the $(\Ev,s)$ marginal), so
constrained ERM transfers without a nuisance-induced gap.
\item[(iii)] \textbf{Necessity.} The global $R_{\mathrm{te}}$-minimizer equals $f^\star\in
\mathcal{F}_{\mathrm{inv}}$ $P_{\mathrm{te}}$-a.s.; restricting to $\mathcal{F}_{\mathrm{inv}}$
loses no optimality.
\end{enumerate}
Moreover the counterfactual penalty
$\mathcal{L}_{\mathrm{cf}}(f)=\mathbb{E}_{\nuis'}\|f(g_{\mathrm{v}}(\Ev,\nuis),s)-f(g_{\mathrm{v}}(\Ev,\nuis'),s)\|^2$
satisfies $\mathcal{L}_{\mathrm{cf}}(f)=0\iff f\in\mathcal{F}_{\mathrm{inv}}$, so minimizing
$R_{\mathrm{tr}}(f)+\lambda\,\mathcal{L}_{\mathrm{cf}}(f)$ recovers $f^\star$ as the constraint
becomes active.
\end{theorem}
Theorem~\ref{thm:cf} is the formal counterpart of Tables~\ref{tab:onset}--\ref{tab:pixels}: the
intervention restricts the hypothesis class to the causal path, which is simultaneously the
train- and test-optimal solution.

\paragraph{What the theorem does \emph{not} say.} The statement is
necessary-and-sufficient \emph{under the stated SCM and intervention assumptions}, and
those assumptions are strong: the nuisance is known and can be intervened on, the
intervention is faithful (it changes only the nuisance and leaves the event intact),
\emph{every} confounded nuisance is covered (Proposition~\ref{prop:partial}), the audio
is genuinely conditionally independent of the nuisance (Assumption~\ref{as:a1}), and the
hypothesis class and optimizer are rich enough to realize $f^\star$. Where any of these
fails---most importantly when the nuisance is unknown---the guarantee does not transfer,
and Appendix~\ref{sec:limitations} is explicit about this. Under these assumptions the
result is admittedly close to a definitional consequence---constraining a predictor to be
$\nuis$-invariant leaves only the $\nuis$-independent causal mean---and it is a
\emph{population}, exact-realizability statement: we do not give finite-$\lambda$,
finite-sample, or non-convex-optimization identification or generalization guarantees, and
we make no such claims. The value is in tying the failure and its cure to a single, checkable
condition, not in the depth of the proof. The theory is a statement about the idealized
causal structure, not a claim about deployed AV generators. Theorem~\ref{thm:cf} also concerns
conditional-\emph{mean} prediction under squared error; its extension to conditional generative
\emph{distributions} $p(\xa\mid\Ev,s)$---the setting of our flow-matching generator
(Table~\ref{tab:gen})---is supported \emph{empirically} rather than covered by the theorem,
although the same nuisance-invariance argument goes through for any strictly proper scoring rule
(Appendix~\ref{app:proofs}, Remark~\ref{rem:scoring}). Finally, the partial-augmentation behavior of
Section~\ref{sec:experiments} is explained by:

\begin{proposition}[Partial augmentation leaves a residual shortcut]\label{prop:partial}
Let $\nuis=(\nuis_1,\dots,\nuis_L)$ be conditionally independent given $\Ev$ and each perfectly
confounded, $\nuis_l=\pi_l(\Ev)$. Enforcing CF-invariance only on $S\subseteq\{1,\dots,L\}$
yields predictors that may still depend on $\nuis_{S^c}$; the resulting OOD excess risk equals the
event-explained audio variance recoverable from $\nuis_{S^c}$, and is zero iff $S$ contains every
confounded nuisance.
\end{proposition}
\section{Experiments}
\label{sec:experiments}

We isolate the statistical question in settings where the causal structure, and
hence the ground-truth robustness, is known exactly. All models predict audio
$\xa$ from video $\xv$ and scene $s$; we report mean-squared audio error
in-distribution (IID, with the train-time confounding intact) and under the
counterfactual shift (OOD-cf, appearance decoupled from event). We compare a
\textbf{direct} model (audio reads the full video; the cross-attention
analogue), a \textbf{common-cause bottleneck}, an unsupervised
\textbf{shared/private disentanglement} (audio from the shared code only), and
\textbf{counterfactual consistency (CC+CF)}.

\subsection{The shortcut is real, and (in this setting) sharpest near perfect confounding}

Table~\ref{tab:onset} sweeps the training confounding strength
$\rho_{\mathrm{spur}}$ in a feature-vector SCM ($K{=}8$ events, $8$ nuisance
values). The direct model is essentially perfect IID at every
$\rho_{\mathrm{spur}}$, but its OOD-cf error rises and then jumps
catastrophically as confounding approaches~$1$: a few decorrelated training
examples are enough for it to learn the causal path, but under \emph{perfect}
confounding it has no incentive to, and fails by $\sim\!140\times$
(Figure~\ref{fig:onset}). The bottleneck
tracks the direct model and also collapses. Counterfactual consistency stays
flat and low across the entire sweep, at the cost of a small constant IID
overhead.

\begin{figure}[t]
\centering
\includegraphics[width=0.5\linewidth]{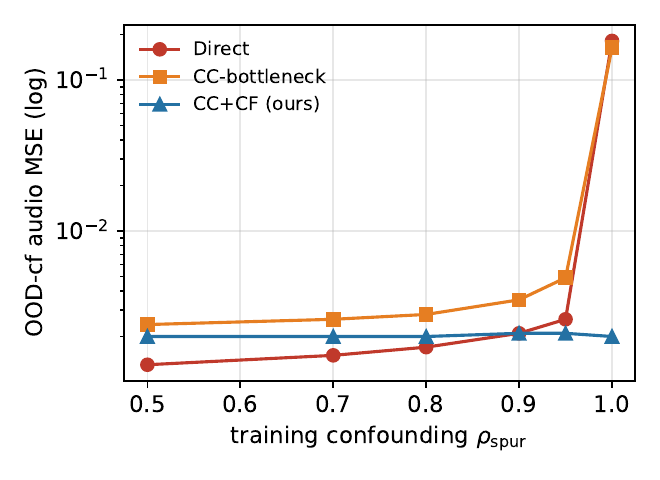}
\caption{Out-of-distribution audio error vs.\ training confounding strength
(log scale; same data as Table~\ref{tab:onset}). The shortcut failure of the
direct and bottleneck models is a sharp phase transition at
$\rho_{\mathrm{spur}}\!\to\!1$; counterfactual consistency is flat and robust.}
\label{fig:onset}
\end{figure}

\begin{table}[t]
\centering
\caption{Audio MSE vs.\ training confounding strength $\rho_{\mathrm{spur}}$
(feature-vector SCM). OOD-cf decouples appearance from event. The direct and
bottleneck models fail catastrophically as confounding approaches perfect in this setting;
counterfactual consistency is robust throughout.}
\label{tab:onset}
\begin{tabular}{lcccccc}
\toprule
& \multicolumn{2}{c}{Direct} & \multicolumn{2}{c}{CC-bottleneck} & \multicolumn{2}{c}{CC+CF (ours)} \\
\cmidrule(lr){2-3}\cmidrule(lr){4-5}\cmidrule(lr){6-7}
$\rho_{\mathrm{spur}}$ & IID & OOD-cf & IID & OOD-cf & IID & OOD-cf \\
\midrule
0.50 & 0.0012 & 0.0013 & 0.0023 & 0.0024 & 0.0019 & 0.0020 \\
0.70 & 0.0012 & 0.0015 & 0.0023 & 0.0026 & 0.0020 & 0.0020 \\
0.80 & 0.0012 & 0.0017 & 0.0022 & 0.0028 & 0.0020 & 0.0020 \\
0.90 & 0.0012 & 0.0021 & 0.0021 & 0.0035 & 0.0020 & 0.0021 \\
0.95 & 0.0011 & 0.0026 & 0.0020 & 0.0049 & 0.0020 & 0.0021 \\
1.00 & 0.0010 & \textbf{0.1816} & 0.0016 & \textbf{0.1637} & 0.0019 & \textbf{0.0020} \\
\bottomrule
\end{tabular}
\end{table}

\subsection{The failure is slice-local, and the fix has no collateral cost}

Real datasets confound only \emph{some} event--appearance pairs. Table~\ref{tab:partial}
confounds half of the event classes perfectly and leaves the rest decoupled. The
direct model fails only on the confounded slice ($40\times$ worse) and is
perfect on the clean slice; counterfactual consistency is low on both, fixing
the shortcut exactly where it exists without harming unconfounded events. The
objective thus costs little: it fixes the confounded slice at a small in-distribution price
(a modest rise in mel-MSE that barely dents event accuracy; Table~\ref{tab:real}), and does
no harm on unconfounded events.

\begin{table}[t]
\centering
\caption{Slice-local confounding: half the event classes are perfectly
confounded, half are clean. Audio MSE under the OOD-cf shift, split by slice.}
\label{tab:partial}
\begin{tabular}{lcc}
\toprule
Model & OOD-cf (confounded slice) & OOD-cf (clean slice) \\
\midrule
Direct        & 0.0488 & 0.0012 \\
CC+CF (ours)  & \textbf{0.0015} & 0.0016 \\
\bottomrule
\end{tabular}
\end{table}

\paragraph{On genuine pixels, and in temporal video.} The mechanism is not an artifact of abstract feature vectors. On procedurally rendered pixel video (event $=$ shape arrangement, nuisance $=$ colour) the direct model's OOD-cf error explodes by $100\times$ and a bottleneck and an unsupervised shared/private factorization fail just as hard, while a recolor-consistency intervention stays robust---architecture alone does not block the shortcut. The same holds when the causal cue is genuinely \emph{temporal}: with the event a cross-frame \emph{motion} and colour a per-frame shortcut, the direct model's motion-accuracy collapses from $1.00$ to $0.52$ while counterfactual consistency stays at $1.00$. Full tables and the video illustration are in Appendix~\ref{app:extra} (Tables~\ref{tab:pixels} and \ref{tab:video}, Figure~\ref{fig:vidstrip}).

\subsection{A faithful shared-prior model fails too}

Our claim that ``sharing a latent is not enough'' should be tested against the
mechanism actually used by recent joint AV models, which align both modalities to
a \emph{shared prior} \citep{liu2025javisdit}, rather than against a generic
bottleneck. We therefore add a faithful stand-in: a single shared code
$z_{\mathrm{s}}=E(\xv)$ from which the audio is generated \emph{and} the video is
reconstructed, so that both modalities align to it. Because $z_{\mathrm{s}}$ must
carry appearance in order to reconstruct the video, the audio branch inherits the
shortcut, and the model collapses under the counterfactual shift exactly like the
direct model (OOD-cf MSE $0.152$ vs.\ direct $0.182$ vs.\ ours $0.002$; full table
in Appendix~\ref{app:ablate}). This stand-in is a small encoder--decoder, not a
full-scale joint-AV diffusion architecture, so we do not claim shared-prior models
\emph{necessarily} fail; the point is the one our theory makes---aligning modalities to a
shared prior is \emph{not sufficient} to remove the shortcut (consistent with
Proposition~\ref{prop:bottleneck})---so the intervention is still needed.

\subsection{Real image pixels and a real audio representation}

We finally test whether the effect survives real image pixels, a real audio
representation, and a real pretrained model. We stress that this is \emph{not} a real
audio--video generation system: the audio is synthesized and the confound constructed,
and it is the closest offline proxy we can build, not a validation on a deployed AV
generator. We use the Colored-MNIST protocol \citep{arjovsky2019invariant}:
real MNIST digits (the event) are tinted with a color (the nuisance) that is
perfectly confounded with the digit class in training and decoupled at test
(Figure~\ref{fig:datamnist}); the
audio target is the $\log$-mel spectrogram of a synthesized instrument tone whose
pitch is set by the digit. The image encoder is a \emph{frozen, ImageNet-pretrained
ResNet-18} \citep{he2016deep}, and only a small audio head is trained.

\begin{figure}[t]
\centering
\includegraphics[width=0.8\linewidth]{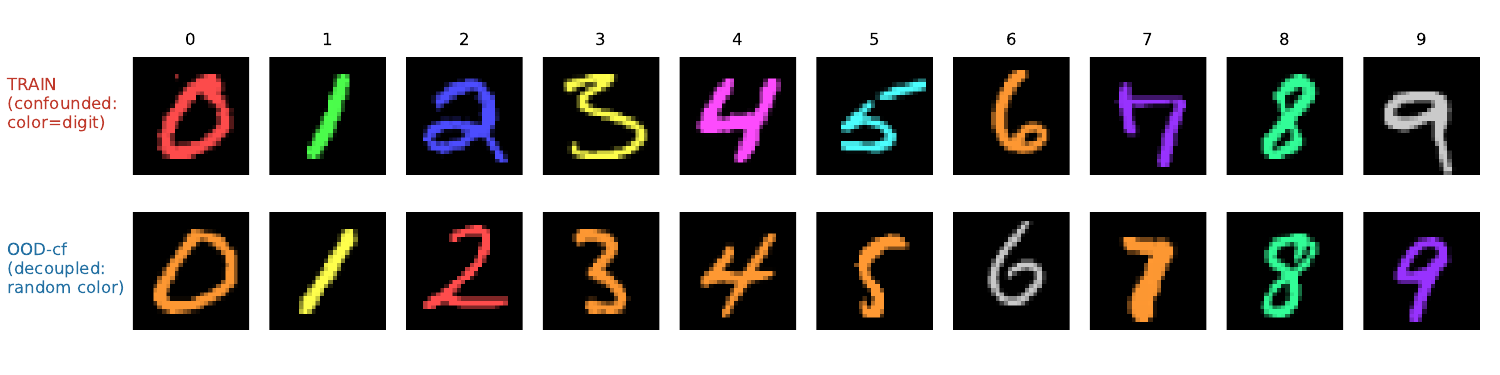}
\caption{The real-modality setting. In training every digit always wears its own
colour (perfect confounding); at test the colour is resampled independently of the
digit. A model that has learned ``colour $\Rightarrow$ sound'' will now render the
wrong tone.}
\label{fig:datamnist}
\end{figure}

Table~\ref{tab:real} shows the same picture on real image pixels: the direct model is perfect in-distribution but its
event-accuracy collapses to near chance ($0.160$, chance $=0.1$) under the
counterfactual color shift---it renders the wrong digit's tone---while
counterfactual consistency stays robust ($0.930$), a $\sim\!6\times$ improvement, at a
modest in-distribution cost. Figure~\ref{fig:specfail} makes this concrete: the
direct model's generated spectrogram matches the tone of the digit whose
\emph{colour} it was shown, not the digit itself. The phenomenon and its fix thus survive real image statistics, a real audio
representation, and a real pretrained backbone---though, we reiterate, not yet a real
AV generator.

\begin{table}[t]
\centering
\caption{Real modalities: Colored-MNIST (real digits, real spectrogram audio,
frozen pretrained ResNet-18), $\rho_{\mathrm{spur}}{=}1$. Mel-spectrogram MSE and
event-accuracy (chance $=0.1$; mean over $3$ seeds, std $\le0.03$ on MSE and
$\le0.008$ on accuracy). The color shortcut is fully exploited by the
direct model and blocked by counterfactual consistency.}
\label{tab:real}
\begin{tabular}{lcccc}
\toprule
& \multicolumn{2}{c}{IID} & \multicolumn{2}{c}{OOD-cf} \\
\cmidrule(lr){2-3}\cmidrule(lr){4-5}
Model & mel-MSE & event-acc & mel-MSE & event-acc \\
\midrule
Direct        & 0.001 & 1.00 & 2.359 & 0.160 \\
CC+CF (ours)  & 0.037 & 0.98 & \textbf{0.105} & \textbf{0.930} \\
\bottomrule
\end{tabular}
\end{table}

\subsection{Is the consistency term doing anything a simple augmentation cannot?}
\label{sec:baselines}

The sharpest test of our objective is supervised counterfactual \emph{augmentation}: if one can synthesize $\xv'=\mathrm{do}(\nuis:=\nuis')$ and knows it carries the same audio, one can simply train on it. We report this honestly. \textbf{With full labels the two are statistically indistinguishable} ($0.0012$ vs.\ $0.0011$ OOD-cf MSE, $5$ seeds), so we do \emph{not} claim the consistency loss as the contribution. What the comparison establishes is the paper's thesis: every method that \emph{intervenes on the nuisance} closes the gap ($\sim\!0.001$), while every method that merely reshapes the representation---the bottleneck, unsupervised shared/private factorization, the shared-prior model, and representation-level consistency ($0.0140$, $12\times$ worse)---does not. The two do diverge once labels are scarce: the consistency term is label-free and keeps improving from unlabeled video, a consistent ${\sim}1.5\times$ advantage as the labeled fraction falls to $0.02$. Full table, label-scarcity curve and discussion are in Appendix~\ref{app:extra} (Table~\ref{tab:baselines}, Figure~\ref{fig:labelscarce}).

\paragraph{When is the shortcut taken?} A shortcut is only attractive when it is \emph{easier} than the causal path, so the phenomenon is conditional. Sweeping the causal-cue strength shows the direct model's counterfactual penalty collapsing smoothly from $263\times$ to $2.7\times$ as the cue becomes easy to read, while the counterfactual objective stays robust and nearly free where it is unnecessary. A Moving-MNIST test with \emph{real} digit sprites confirms the falsifiable prediction: the shortcut appears only in the hard-motion regime (direct OOD accuracy $0.21$ vs.\ ours $0.995$) and there is nothing to fix when the motion is easy (both $\approx\!1.0$). Full sweep and table in Appendix~\ref{app:extra} (Figure~\ref{fig:cue}, Table~\ref{tab:mm}).

\subsection{A real video-to-audio generator is not invariant to sound-irrelevant edits}
\label{sec:realv2a}

Our controlled settings construct the confound to isolate the shortcut. On a \emph{real,
deployed} model we cannot construct one, so instead of isolating a shortcut we test a
necessary condition: is the model invariant to input edits that do not change what an event
sounds like? We run MMAudio \citep{cheng2025mmaudio}, a pretrained video-to-audio generator,
on $24$ real videos (UCF101 \citep{soomro2012ucf101} actions with characteristic sounds, plus
the model's own example clips). For each video and $3$ seeds we generate audio from the
original and from edited versions at a fixed generation seed (so any change is attributable
to the edit) and measure the shift of the generated audio in CLAP space \citep{wu2023large},
as a distance and as a zero-shot top-1 \emph{sound-class flip}. Two edits change only
\emph{appearance} while preserving motion (grayscale, colour-channel rotation); a $+2\%$
brightness edit is the robustness reference and time reversal is a causal (motion-changing)
reference. Crucially we add a \emph{generic out-of-distribution control}: additive Gaussian
pixel noise, which carries no event information and is not a colour change.

The model is clearly \emph{not} invariant to the appearance edits (Figure~\ref{fig:realv2a},
Table~\ref{tab:flip}): grayscale and colour rotation flip the top-1 sound class $62\%$ and
$54\%$ of the time versus $22\%$ for the minimal edit (non-overlapping CIs), shifting the
CLAP embedding $\sim\!5$--$6\times$ as far. Recolouring or graying a video substantially
changes the sound this real model generates. \textbf{But the appearance edits are not special
once perturbation magnitude is controlled.} Calibrating a semantically-empty Gaussian-noise
control to the same per-video PSNR as each appearance edit (grayscale $\approx19.9$\,dB,
$\sigma{=}0.10$; colour $\approx15.6$\,dB, $\sigma{=}0.17$) and re-measuring, the matched
noise flips the sound class as often or more: grayscale $62\%$ vs.\ $60\%$ (overlapping CIs),
colour $54\%$ vs.\ $\mathbf{65\%}$. The model's sensitivity tracks perturbation size, not
appearance content. We therefore read the real-model result as \emph{non-invariance} to
sound-irrelevant edits (a necessary symptom consistent with shortcutting) and \emph{do not}
claim it isolates a colour/appearance shortcut; that isolation needs the controlled confound
available only in the synthetic settings above.

\begin{table}[t]
\centering
\caption{Sensitivity of a real V2A generator (MMAudio) to input edits, over $24$ real videos
$\times 3$ seeds ($95\%$ bootstrap CIs): sound-class flip rate and mean CLAP distance of the
generated audio. Appearance edits change the sound far more than the minimal edit, but a
\emph{magnitude-matched} Gaussian-noise control (same per-video PSNR) moves it at least as
much: the effect is broad input-brittleness, not an isolated colour shortcut.}
\label{tab:flip}
\begin{tabular}{lcc}
\toprule
Input edit & sound-class flip rate & mean CLAP distance \\
\midrule
tiny ($+2\%$ brightness), robust ref. & 22\% \,[13,32] & 0.060\,[0.04,0.08] \\
time reversal (causal)   & 44\% \,[33,57] & 0.276\,[0.24,0.32] \\
\midrule
colour rotation (nuisance) & 54\% \,[43,65] & 0.357\,[0.32,0.40] \\
grayscale (nuisance)     & 62\% \,[51,74] & 0.342\,[0.30,0.38] \\
\midrule
Gaussian noise, generic ($\sigma{=}0.15$)   & 60\% \,[47,71] & 0.452\,[0.42,0.49] \\
Gaussian noise, magnitude-matched to grayscale & 60\% \,[49,71] & 0.400\,[0.36,0.44] \\
Gaussian noise, magnitude-matched to colour & \textbf{65\%}\,[54,76] & \textbf{0.472}\,[0.43,0.51] \\
\midrule
different video (ceiling) & --- & 0.673 \\
\bottomrule
\end{tabular}
\end{table}

In short, the real-model experiment confirms a necessary symptom---a deployed V2A generator
is far from invariant to edits that leave the event unchanged---but, honestly, cannot
attribute it specifically to a colour/appearance shortcut, because a magnitude-matched,
semantically-empty noise perturbation is just as disruptive. The clean isolation of the
shortcut, and of its fix, lives in the controlled studies; the real model shows the symptom
is present at scale.

\paragraph{Does the intervention transfer to the real model?} Finally, we ask whether the
counterfactual objective can \emph{reduce} the observed nuisance sensitivity in this real
model. Without any real audio, we fine-tune MMAudio's velocity network for $1500$ steps on
$18$ held-out videos with a recolor-consistency term
($\|v_\theta(x_t,t\,|\,\text{recolor})-v_{\text{frozen}}(x_t,t\,|\,\text{video})\|^2$)
anchored by distillation to the frozen model on the original video, with $x_t$ sampled along
the frozen model's own generation path. To gauge stability we fine-tune $3$ times with
independent seeds and evaluate every model, and the frozen baseline, on a common held-out
set ($12$ videos $\times 2$ seeds) so ``before'' and ``after'' are strictly matched. The
intervention \emph{does} reduce the grayscale sensitivity it targets, reproducibly and
without collapsing the model: the grayscale flip rate falls from $70.8\%$ (frozen) to
$54.2\%\!\pm\!3.4$, while the different-video ceiling is preserved ($0.67$ to $0.69$), so the
model still produces distinct, video-dependent audio. \textbf{But the effect is not the
surgical appearance-only correction one might hope for.} On the same set the colour-rotation
flip is essentially unchanged ($62.5\%$ to $63.9\%\!\pm\!2.0$), while the generic
Gaussian-noise sensitivity ($83.3\%$ to $70.8\%\!\pm\!0.0$) and even the \emph{causal}
time-reversal sensitivity ($58.3\%$ to $43.1\%\!\pm\!1.9$) fall by a similar margin: the
lightweight fine-tune broadly \emph{dampens} sensitivity to input perturbations rather than
selectively erasing an appearance shortcut, as expected if the deployed model's
non-invariance is broad rather than colour-specific. We therefore claim only what the data
support: the objective transfers far enough to \emph{measurably and stably reduce} the
targeted appearance sensitivity while keeping the model functional, but it neither eliminates
that sensitivity nor isolates it from the model's general input-brittleness. Cleanly
de-biasing a large pretrained generator (separating appearance from causal sensitivity, via
real-audio supervision or training from scratch) is left to future work.

\subsection{The shortcut, and its fix, appear in a generative model}

To confirm the phenomenon is not an artifact of audio \emph{regression}, we repeat it with a
genuine conditional \emph{generator}: a flow-matching model \citep{lipman2023flow} that
samples audio $\xa$ from video features, trained directly or with the
counterfactual-consistency term. Alongside sample error, Table~\ref{tab:gen} reports the
\emph{event-accuracy} of the sampled audio (does it match the correct event's spectrum). The
direct generator is perfect in-distribution but its event-accuracy collapses to chance under
the counterfactual shift ($0.13$, chance $0.125$), generating the wrong event's sound (the
one implied by the shortcut appearance), whereas counterfactual consistency keeps it at
$1.0$. A linear probe explains the gap: the nuisance material is fully decodable from the
direct and bottleneck representations (accuracy $1.0$) but close to chance ($0.24$) under
counterfactual consistency, with event-decodability still $1.0$, so the representation is
reorganized around the cause. Finally, with several individually-sufficient nuisances
robustness is all-or-nothing: any single un-augmented confounder preserves a complete
shortcut (Proposition~\ref{prop:partial}). The probe figure and the partial-augmentation
curve are in Appendix~\ref{app:extra} (Figure~\ref{fig:probe}).

\begin{table}[t]
\centering
\caption{Conditional \emph{generation} (flow matching), $\rho_{\mathrm{spur}}{=}1$.
We report generated-sample audio MSE and the event-accuracy of the sampled audio
(mean$\pm$std over $3$ seeds). The direct generator synthesizes the
shortcut-implied wrong event's sound OOD.}
\label{tab:gen}
\begin{tabular}{lcccc}
\toprule
& \multicolumn{2}{c}{IID} & \multicolumn{2}{c}{OOD-cf} \\
\cmidrule(lr){2-3}\cmidrule(lr){4-5}
Model & MSE & event-acc & MSE & event-acc \\
\midrule
Direct generator      & 0.0016 & 1.00 & $0.171{\pm}0.002$ & $0.130{\pm}0.000$ \\
CC+CF generator (ours)& 0.0016 & 1.00 & $\mathbf{0.0016{\pm}0.000}$ & $\mathbf{1.00{\pm}0.00}$ \\
\bottomrule
\end{tabular}
\end{table}

\section{Conclusion}
We showed that joint audio--video generators learn a visual shortcut, that
sharing a latent does not remove it, and that---under the stated SCM assumptions---
\emph{counterfactual invariance} is necessary and sufficient to identify the causal
predictor, an intervention on the nuisance rather than a change of architecture. This
is a population, exact-realizability characterization; empirically the resulting
consistency objective blocks the shortcut in controlled settings, including on pixel
video and in a conditional generator. The contribution reframes robust AV generation as a
question of intervention, and isolates the unknown-nuisance regime as the key
remaining challenge.

\newpage
\section*{Reproducibility Statement}
All theoretical claims are stated with their assumptions in Section~\ref{sec:theory} and
proved in full in Appendix~\ref{app:proofs}. The structural causal models, architectures,
training objectives, and hyperparameters for every experiment are specified in
Appendix~\ref{app:details}, which also documents the seed protocol and how means and standard
deviations are reported (multi-seed tables and the matched-eval protocol for the real-model
fine-tune). All datasets used are either procedurally generated from the described SCMs or
public (MNIST; UCF101 \citep{soomro2012ucf101}); the real video-to-audio experiments use the
publicly released MMAudio \citep{cheng2025mmaudio} and CLAP \citep{wu2023large} checkpoints.
Anonymized code to reproduce all controlled experiments and figures is provided with the
submission.

\bibliography{iclr2026_conference}
\bibliographystyle{iclr2027_conference}

\appendix

\section{Positioning and Scope}
\label{app:positioning}

\paragraph{What is and is not new here.} Consistency regularization under a known
transformation is a standard tool, and we do not claim it as a novel mechanism; the
individual ingredients---counterfactual invariance \citep{veitch2021counterfactual},
invariant risk minimization \citep{arjovsky2019invariant}, causal latent-variable models
\citep{yang2021causalvae}---all predate us. Our contribution is (i) to identify and
formalize the visual shortcut as the failure mode of joint AV generation, (ii) to show,
in controlled settings where the ground truth is known exactly, that the field's current
remedy---aligning modalities to a shared latent---provably and empirically does
\emph{not} address it, (iii) to establish precisely when the intervention is required
(and, equally, when it is not), and (iv) to isolate the unknown-nuisance regime as the
concrete obstacle to deploying it. The shift is one of \emph{burden}: from
representation design to intervention.

\paragraph{Scope.} This is a controlled causal study. Our most realistic settings use
real image pixels, a real audio representation, and a pretrained backbone, but the
audio is synthesized and the confounding is constructed---because measuring a shortcut
counterfactually \emph{requires} knowing the confound. We therefore establish a
mechanism and its conditions, not a validated fix for any deployed AV system;
transferring the objective to recorded in-the-wild AV corpora and full-scale joint AV
generators is the natural next step, discussed in Section~\ref{sec:limitations}.

\section{Limitations and the Unknown-Nuisance Problem}
\label{sec:limitations}

\textbf{The main limitation concerns breadth of real-model evidence.} On a real, pretrained
video-to-audio generator we can only establish a \emph{necessary symptom}, not an isolated
shortcut: the model is far from invariant to sound-irrelevant input edits
(Section~\ref{sec:realv2a}), which is consistent with shortcutting but---because a generic
Gaussian-noise control moves the generated sound at least as much as the appearance edits---
cannot be attributed specifically to a colour/appearance shortcut. Isolating a specific shortcut
requires the constructed confound we have only in the synthetic settings, where the
\emph{controlled} evidence necessarily uses synthesized audio (measuring a shortcut
counterfactually requires knowing the confound). Our recolor-consistency fine-tune of the real
model is likewise honest about what it shows: across $3$ seeds on a matched evaluation set it
\emph{stably reduces the targeted grayscale sensitivity} ($70.8\%\!\to\!54.2\%$) without
collapsing the model (the different-video semantic ceiling is preserved), but it does not
isolate appearance---colour flip is essentially unchanged and the generic-noise and causal
time-reversal sensitivities fall by a similar margin, so the intervention broadly \emph{dampens}
input sensitivity rather than surgically removing an appearance shortcut. It thus neither
demonstrates nor removes a real-model shortcut; that, and a survey beyond one model and global
colour edits, remain open.
Closing this gap---fully removing the shortcut inside a deployed joint-AV generator (with
real-audio supervision or training from scratch), extending the intervention across models
and to material/background edits, and building an appearance-preserving
event-counterfactual benchmark on recorded AV keyed to physical-commonsense benchmarks
\citep{cui2026joint,li2026benchmarking}---is the most important direction this work leaves
open. Three scope conditions are intrinsic
rather than incidental. First, the intervention is decisive mainly under
near-perfect (possibly slice-local) confounding; establishing that real AV
corpora contain such slices is an empirical prerequisite, for which existing
benchmarks provide encouraging but indirect evidence. Second, the shortcut is taken
only when the causal cue is hard to read relative to the nuisance
(Section~\ref{sec:cue}): where the event is visually obvious---as with a large,
unambiguous motion---a direct model is already robust and our objective neither helps
nor hurts. Which real AV events are visually subtle enough to be shortcut is an open
empirical question; the near-zero cost of the objective, however, makes leaving it on
the safe default. Third, and most
important, the practical recipe assumes \emph{every} confounding nuisance is
\emph{identifiable and augmentable}: Proposition~\ref{prop:partial} and its
experiment show robustness is all-or-nothing, so missing even one confounder
leaves a complete shortcut, and our unsupervised shared/private baseline shows
that representation learning alone does not recover robustness. Discovering
unknown nuisances and intervening on all of them---rather than assuming they are
given---is, in our view, the central open problem for causally robust
audio--video generation, and the most promising direction this work points to.

\section{Extended Related Work}
\label{app:relwork}

Our contribution sits at the intersection of several literatures; we expand on each here and
mark where we differ.

\paragraph{Joint audio--video and video-to-audio generation.} Recent systems couple the two
modalities either by direct cross-modal attention \citep{ruan2023mm,xing2024seeing} or, more
recently, by aligning both to a shared spatio-temporal prior
\citep{liu2025javisdit,hacohen2026ltx}; video-to-audio models such as
Diff-Foley \citep{luo2023difffoley} and MMAudio \citep{cheng2025mmaudio} map video (and text)
to audio through a latent diffusion / flow-matching backbone. Physical-commonsense benchmarks
report that current joint and V2A models are sensitive to exactly the visual factors they
should ignore \citep{cui2026joint,li2026benchmarking}, providing external evidence that the
shortcut we isolate occurs in practice. Unlike these works, which propose architectures or
benchmarks, we give a causal account of \emph{why} the coupling produces a shortcut and what
provably removes it.

\paragraph{Shortcut learning and spurious correlation.} Shortcut learning is a general failure
of deep models \citep{geirhos2020shortcut}; remedies include reweighting worst-case groups
\citep{sagawa2020group} and de-biasing from a deliberately biased reference model
\citep{nam2020learning}. These operate on classification with known or inferred group labels;
we instead exploit the generative AV structure, where a \emph{feasible} nuisance intervention
(recolouring, restyling) replaces the need for group annotations.

\paragraph{Invariance and domain generalization.} IRM \citep{arjovsky2019invariant}, V-REx
\citep{krueger2021rex} and GroupDRO \citep{sagawa2020group} seek predictors invariant across
training environments; DomainBed \citep{gulrajani2021search} shows their gains over ERM are
often fragile, and adversarial feature alignment \citep{ganin2016domain} is unstable. As our
Appendix~\ref{app:review} comparison makes explicit, these methods trade one requirement---many
decorrelating environments---for the one we use---a single known, faithful counterfactual---and
the latter is available precisely in the AV setting we study.

\paragraph{Counterfactual invariance and causal representation learning.} Counterfactual
invariance \citep{veitch2021counterfactual} and neural structural causal models
\citep{yang2021causalvae} provide the conceptual backbone; the impossibility of unsupervised
disentanglement without inductive bias \citep{locatello2019challenging} explains why a
shared/private factorization cannot, on its own, separate event from appearance. Our
contribution is to instantiate counterfactual invariance for the specific audio--video
common-cause SCM, to prove that the shared-latent remedy is insufficient there, and to
identify the unknown-nuisance regime as the barrier to deployment.

\section{Full Proofs}
\label{app:proofs}

Throughout, $(\Ev,s,\nuis)$ are the latent factors, $\xv=g_{\mathrm{v}}(\Ev,\nuis)$,
$\mu(\Ev,s)=\mathbb{E}[\xa\mid\Ev,s,\nuis]$ (Assumption~\ref{as:a1}), and by
Assumption~\ref{as:a2} every measurable predictor factors as
$f(\xv,s)=\tilde f(\Ev,\nuis,s)$. We write $R_D(f)=\mathbb{E}_D\|f(\xv,s)-\xa\|^2$.
$P_{\mathrm{tr}}$ denotes the confounded training law and $P_{\mathrm{te}}$ the
counterfactual test law, in which $\nuis\perp(\Ev,s)$ and the $(\Ev,s)$ marginal is
unchanged. $\Ev$ is discrete with finite support.

\begin{lemma}[Risk decomposition]\label{lem:decomp}
For any measurable $f$ and any law $D$ over $(\Ev,s,\nuis)$,
\[
R_D(f)=\mathbb{E}_D\big\|\tilde f(\Ev,\nuis,s)-\mu(\Ev,s)\big\|^2
+\mathbb{E}_D\big[\operatorname{tr}\operatorname{Cov}(\xa\mid \Ev,s)\big],
\]
and the second term does not depend on $f$.
\end{lemma}
\begin{proof}
Write $\|\tilde f-\xa\|^2=\|\tilde f-\mu\|^2+\|\mu-\xa\|^2+2\langle \tilde f-\mu,\ \mu-\xa\rangle$.
Conditioning the cross term on $(\Ev,\nuis,s)$ and using
$\mathbb{E}[\xa\mid \Ev,\nuis,s]=\mu(\Ev,s)$ (Assumption~\ref{as:a1}) together with the
$(\Ev,\nuis,s)$-measurability of $\tilde f-\mu$ gives
$\mathbb{E}\langle \tilde f-\mu,\mu-\xa\rangle
=\mathbb{E}\langle \tilde f-\mu,\ \mu-\mathbb{E}[\xa\mid \Ev,\nuis,s]\rangle=0$.
The remaining term $\mathbb{E}\|\mu-\xa\|^2$ equals
$\mathbb{E}[\operatorname{tr}\operatorname{Cov}(\xa\mid \Ev,s)]$ by
Assumption~\ref{as:a1} and is free of $f$.
\end{proof}

Lemma~\ref{lem:decomp} means minimizing risk is exactly matching $\mu$, and that the
causal predictor $f^\star=\mu(\Ev,s)$ is the Bayes predictor under \emph{any} law.

\begin{proof}[Proof of Proposition~\ref{prop:blind}]
\textbf{(i)} Under perfect confounding $\nuis=\pi(\Ev)$ $P_{\mathrm{tr}}$-a.s.\ with $\pi$
injective, we have $\pi^{-1}(\nuis)=\Ev$ $P_{\mathrm{tr}}$-a.s., hence
$f_{\mathrm{sc}}=\mu(\pi^{-1}(\nuis),s)=\mu(\Ev,s)=f^\star$ $P_{\mathrm{tr}}$-a.s.
Two predictors that agree a.s.\ have equal risk, so
$R_{\mathrm{tr}}(f_{\mathrm{sc}})=R_{\mathrm{tr}}(f^\star)$; by
Lemma~\ref{lem:decomp} this is the Bayes risk, so both are training-risk minimizers.

\textbf{(ii)} Under $P_{\mathrm{te}}$ set $\Ev':=\pi^{-1}(\nuis)$. Since
$\nuis\perp(\Ev,s)$ and $\pi$ is a bijection onto its range, $\Ev'\perp(\Ev,s)$ and
$\Ev'$ has the marginal law of $\Ev$. Applying Lemma~\ref{lem:decomp} to both
predictors under $P_{\mathrm{te}}$, the $f$-free terms cancel and
\[
R_{\mathrm{te}}(f_{\mathrm{sc}})-R_{\mathrm{te}}(f^\star)
=\mathbb{E}\big\|\mu(\Ev',s)-\mu(\Ev,s)\big\|^2-0
=\mathbb{E}\big\|\mu(\Ev',s)-\mu(\Ev,s)\big\|^2\ \ge 0 .
\]
If $\mu(\cdot,s)$ is non-constant on a set of positive measure, then since
$\Ev'\perp\Ev$ with the same marginal, $\Pr[\mu(\Ev',s)\neq\mu(\Ev,s)]>0$ and the gap
is strictly positive.
\end{proof}

\begin{proof}[Proof of Proposition~\ref{prop:bottleneck}]
Assume perfect confounding, so $\sigma(\nuis)=\sigma(\Ev)$ under $P_{\mathrm{tr}}$.

\emph{$z=\nuis$ is risk-optimal.} Because $\nuis$ determines $\Ev$, the predictor
$(\nuis,s)\mapsto\mu(\pi^{-1}(\nuis),s)$ equals $\mu(\Ev,s)$ $P_{\mathrm{tr}}$-a.s., so by
Lemma~\ref{lem:decomp} an encoder $E(\xv)=\nuis$ with an appropriate audio decoder
attains the Bayes training risk.

\emph{$z=\nuis$ is rate-minimal among risk-optimal encoders.} Suppose an encoder
$z=E(\xv)$ admits a decoder attaining the Bayes training risk; by
Lemma~\ref{lem:decomp} its decoder output equals $\mu(\Ev,s)$ a.s. If the event is
audio-identifiable on a positive-measure set of scenes---$\forall\,\Ev\neq\Ev'$,
$P_s[\mu(\Ev,s)\neq\mu(\Ev',s)]>0$---then the map $s\mapsto\mu(\Ev,s)$ pins down $\Ev$, so
$z$ (which reproduces $\mu(\cdot,s)$ across scenes) determines $\Ev$ a.s., giving
$I(z;\xv)\ge H(\Ev)$. Meanwhile $I(\nuis;\xv)=H(\nuis)=H(\Ev)$ because $\pi$ is a
bijection. Hence $z=\nuis$ attains the minimum achievable rate among risk-optimal
encoders, i.e.\ the pair $(\text{risk},\text{rate})$ at $z=\nuis$ is Pareto-optimal.

\emph{Conclusion.} For any budget $\beta\ge H(\Ev)$, the program
$\min_E R_{\mathrm{tr}}$ s.t.\ $I(z;\xv)\le\beta$ admits $E(\xv)=\nuis$ as a
minimizer; no such objective can prefer $z=\Ev$ over $z=\nuis$, since the two are
indistinguishable in both risk and rate. The predictor induced by $z=\nuis$ is
exactly $f_{\mathrm{sc}}$, whose OOD excess risk is strictly positive by
Proposition~\ref{prop:blind}. Finally, for a shared/private split
$\xv\mapsto(z_{\mathrm{s}},z_{\mathrm{v}})$ with audio decoded from $z_{\mathrm{s}}$
only, the audio term of the objective is exactly the program above, and the video
reconstruction term can be satisfied by $z_{\mathrm{v}}$; hence
$z_{\mathrm{s}}=\nuis$ again lies among the minimizers.
\end{proof}

\begin{proof}[Proof of Theorem~\ref{thm:cf}]
\textbf{(i) Identification.} Let $f\in\mathcal{F}_{\mathrm{inv}}$. By definition
$\tilde f(\Ev,\nuis,s)$ does not depend on $\nuis$, so $f=h(\Ev,s)$ for some
measurable $h$. By Lemma~\ref{lem:decomp},
$R_{\mathrm{tr}}(f)=\mathbb{E}_{\mathrm{tr}}\|h(\Ev,s)-\mu(\Ev,s)\|^2+\mathrm{const}$,
which is minimized if and only if $h=\mu$ $P_{\mathrm{tr}}$-a.s.; hence
$\arg\min_{\mathcal{F}_{\mathrm{inv}}}R_{\mathrm{tr}}=f^\star$. For the second claim,
Lemma~\ref{lem:decomp} applied under $P_{\mathrm{te}}$ gives, for \emph{any}
measurable $g$, $R_{\mathrm{te}}(g)\ge\mathrm{const}$ with equality iff
$\tilde g=\mu$ a.s.; $f^\star$ attains it, so $f^\star$ is the global
$R_{\mathrm{te}}$-minimizer.

\textbf{(ii) Gap closure.} For $f=h(\Ev,s)\in\mathcal{F}_{\mathrm{inv}}$,
Lemma~\ref{lem:decomp} gives
$R_{D}(f)=\mathbb{E}_{D(\Ev,s)}\|h-\mu\|^2+\mathbb{E}_{D}[\operatorname{tr}\operatorname{Cov}(\xa\mid\Ev,s)]$,
and both terms depend on $D$ only through the $(\Ev,s)$ marginal, which
$P_{\mathrm{tr}}$ and $P_{\mathrm{te}}$ share. Hence
$R_{\mathrm{tr}}(f)=R_{\mathrm{te}}(f)$.

\textbf{(iii) Necessity.} By (i), any global $R_{\mathrm{te}}$-minimizer $g$ satisfies
$\tilde g=\mu$ $P_{\mathrm{te}}$-a.s. Since under $P_{\mathrm{te}}$ we have
$\nuis\perp(\Ev,s)$ with $\nuis$ ranging over its full support, $\tilde g(\Ev,\nuis,s)=\mu(\Ev,s)$
for a.e.\ $(\Ev,\nuis,s)$, i.e.\ $\tilde g$ is a.s.\ constant in $\nuis$; thus
$g\in\mathcal{F}_{\mathrm{inv}}$ up to a null set, and $f^\star\in\mathcal{F}_{\mathrm{inv}}$.
Restricting to $\mathcal{F}_{\mathrm{inv}}$ therefore excludes no optimal predictor.

\textbf{Penalty equivalence.} $\mathcal{L}_{\mathrm{cf}}(f)=
\mathbb{E}_{\nuis'}\|\tilde f(\Ev,\nuis,s)-\tilde f(\Ev,\nuis',s)\|^2=0$ iff
$\tilde f(\Ev,\cdot,s)$ is a.s.\ constant on the support of $\nuis$, i.e.\ iff
$f\in\mathcal{F}_{\mathrm{inv}}$. Consequently $\min_f R_{\mathrm{tr}}(f)+\lambda
\mathcal{L}_{\mathrm{cf}}(f)$ is, in the limit of an active constraint, the program of
(i), whose solution is $f^\star$ whenever $f^\star$ is realizable.
\end{proof}

\begin{proof}[Proof of Proposition~\ref{prop:partial}]
Let $S\subseteq\{1,\dots,L\}$ and let $f$ be invariant to $\nuis_S$. By
Assumption~\ref{as:a2}, $f=h(\Ev,\nuis_{S^c},s)$. If $S^c=\emptyset$ then
$f=h(\Ev,s)\in\mathcal{F}_{\mathrm{inv}}$ and Theorem~\ref{thm:cf}(i) gives zero excess
OOD risk at the optimum. Otherwise pick any $l\in S^c$; by hypothesis
$\nuis_l=\pi_l(\Ev)$ $P_{\mathrm{tr}}$-a.s.\ with $\pi_l$ injective, so the predictor
$h_{\mathrm{sc}}(\Ev,\nuis_{S^c},s):=\mu(\pi_l^{-1}(\nuis_l),s)$ is invariant to
$\nuis_S$, agrees with $f^\star$ $P_{\mathrm{tr}}$-a.s., and therefore attains the Bayes
training risk---so it lies in the constrained argmin. Applying
Proposition~\ref{prop:blind} with the reduced nuisance $\nuis_l$ gives OOD excess risk
$\mathbb{E}\|\mu(\Ev',s)-\mu(\Ev,s)\|^2>0$ with $\Ev'=\pi_l^{-1}(\nuis_l)\perp\Ev$.
Hence the constrained program admits a shortcut solution with strictly positive OOD
excess risk whenever some confounded nuisance is left un-augmented, and zero excess
risk exactly when $S$ contains every confounded nuisance.
\end{proof}

\begin{remark}[Extension to conditional generative distributions]\label{rem:scoring}
Theorem~\ref{thm:cf} is stated for conditional-\emph{mean} prediction under squared error, but
the identification argument does not depend on that choice. Let the model output a conditional
density $q(\cdot\mid\xv,s)$ and be scored by a strictly proper scoring rule
$S(q,\xa)$, e.g.\ the log score. By strict propriety the population risk
$\mathbb{E}[S(q(\cdot\mid\xv,s),\xa)]$ is uniquely minimized at the true conditional
$q^\star(\cdot\mid\xv,s)=p(\xa\mid\Ev,\nuis,s)$, which by Assumption~\ref{as:a1} equals
$p(\xa\mid\Ev,s)$ and is therefore $\nuis$-invariant. Restricting to the counterfactually
invariant class $\mathcal{F}_{\mathrm{inv}}$ (now defined on the conditional law
$q(\cdot\mid g_{\mathrm v}(\Ev,\nuis),s)$) thus loses no optimality and closes the train/test
gap, exactly as in Theorem~\ref{thm:cf}(i)--(iii); the counterfactual-consistency penalty
becomes a divergence between $q(\cdot\mid\xv,s)$ and $q(\cdot\mid\xv',s)$. We use this only to
justify the framing; the generative results (Table~\ref{tab:gen}) are reported empirically.
\end{remark}

\section{Experimental Details}
\label{app:details}

All experiments run on a single NVIDIA RTX A6000; every run completes in minutes.
Optimizer is Adam throughout. Unless stated otherwise the counterfactual weight is
$\lambda=1$ and the training confounding is $\rho_{\mathrm{spur}}=1$.

\subsection{Feature-vector SCM (Tables~\ref{tab:onset}, \ref{tab:partial},
\ref{tab:gen}, \ref{tab:shared}; Figures~\ref{fig:onset}, \ref{fig:probe}).}
\label{app:fvscm}
$K{=}8$ events, $M{=}8$ nuisance values, video dimension $d_{\mathrm{v}}{=}32$, audio
dimension $d_{\mathrm{a}}{=}16$. The video is
$\xv=[\,\gamma\cdot\mathrm{onehot}(\Ev)\ ;\ \mathrm{onehot}(\nuis)\,]W_{\mathrm{v}}+0.05\,\epsilon$
with a fixed random $W_{\mathrm{v}}\in\mathbb{R}^{(K+M)\times d_{\mathrm{v}}}$ scaled by
$1/\sqrt{K+M}$ and \emph{geometry strength} $\gamma{=}0.6$: the causal event signal is
deliberately weaker than the clean one-hot nuisance, which makes the shortcut the path
of least resistance. The audio is
$\xa=(\mathrm{onehot}(\Ev)W_{\mathrm{a}})\cdot(1/s)+0.03\,\epsilon$ with
$W_{\mathrm{a}}\in\mathbb{R}^{K\times d_{\mathrm{a}}}$ and scene distance
$s\sim\mathcal{U}[0.5,1.5]$ (so the scene sets loudness, and audio depends on
$(\Ev,s)$ only). Training draws $\nuis=\Ev \bmod M$ with probability
$\rho_{\mathrm{spur}}$ and uniform otherwise; the OOD-cf test draws $\nuis$ uniformly,
independent of $\Ev$. \emph{Direct}: MLP $[d_{\mathrm{v}}{+}1\to256\to256\to d_{\mathrm{a}}]$
with SiLU. \emph{CC-bottleneck / CC+CF}: encoder $[d_{\mathrm{v}}\to256\to z]$ with
$z{=}6$, decoder $[z{+}1\to256\to d_{\mathrm{a}}]$. Training: lr $2\times10^{-3}$, batch
$256$, $3000$ steps. Evaluation: $4000$ fresh samples. The counterfactual batch
resamples $\nuis$ while holding $(\Ev,s)$ fixed.

\subsection{Slice-local confounding (Table~\ref{tab:partial}).}
Same SCM; the first $4$ of $8$ event classes are perfectly confounded
($\nuis=\Ev$), the remaining classes always receive a uniform nuisance. The test
decouples all classes; we report the median per-sample error separately over the
confounded and clean event classes ($8000$ evaluation samples).

\subsection{Procedural pixel video (Table~\ref{tab:pixels}).}
$32{\times}32$ RGB. Each of $K{=}8$ events is a fixed template of $2$--$4$ Gaussian
blobs ($\sigma{=}0.16$) at event-specific offsets, rendered with random global
translation ($\pm0.3$) and scale ($0.8$--$1.2$); the nuisance is one of $M{=}8$ RGB
colors multiplying the grayscale render. Audio is a per-event $16$-d spectrum scaled by
$1/s$. Encoder: three stride-2 convolutions ($3{\to}32{\to}64{\to}64$, kernel $4$)
followed by a linear map. \emph{Direct} uses a $128$-d code, \emph{bottleneck} a $6$-d
code, \emph{disentangle} $z_{\mathrm{s}}{=}6$ shared and $z_{\mathrm{v}}{=}32$ private
with a mirrored transposed-convolution image decoder trained on reconstruction plus the
audio loss from $z_{\mathrm{s}}$ alone. \emph{CFaug} re-renders the same grayscale with
a random palette color and penalizes the audio discrepancy. Training: lr $10^{-3}$,
batch $128$, $3000$ steps; evaluation on $3000$ samples.

\subsection{Temporal video (Table~\ref{tab:video}).}
$K{=}8$ motion directions (angles $2\pi k/8$), $T{=}8$ frames, $32{\times}32$ RGB. A
Gaussian blob ($\sigma{=}0.13$) starts at a uniformly random position in
$[-0.35,0.35]^2$ and is displaced by $1.1\cdot(t/T)$ along its direction, so the
\emph{start position carries no information about the event} and only the cross-frame
displacement identifies it; the color (one of $8$) is constant within a clip and
confounded with the direction. Audio is a per-direction $16$-d spectrum. The model
stacks time on channels ($3T{=}24$) and applies three stride-2 convolutions
($24{\to}32{\to}64{\to}64$) plus a linear head. Training: lr $10^{-3}$, batch $128$,
$4000$ steps. The counterfactual re-renders the same clip (same motion and start) in a
random color. Motion-accuracy is nearest-neighbour classification of the predicted audio
against the $8$ direction spectra ($2000$ evaluation samples).

\subsection{Real modalities: Colored-MNIST with spectrogram audio (Table~\ref{tab:real}).}
Images: the $60{,}000$ MNIST training digits, tinted by one of $10$ palette colors,
bilinearly resized to $64{\times}64$ and ImageNet-normalized. Audio: for digit $d$ the
tone has fundamental $f_0=220\cdot2^{d/12}$\,Hz with five harmonics of amplitude $1/h$
and envelope $e^{-3t/T}$, $T{=}0.4$\,s at $16$\,kHz; the target is
$\log(1+\mathrm{Mel})$ with $n_{\mathrm{fft}}{=}512$, hop $256$, $32$ mel bands,
flattened to $832$ dimensions. Encoder: torchvision ResNet-18 with ImageNet weights,
final layer removed, \textbf{frozen} (only the head is trained), giving $512$-d
features. Head: MLP $[512\to512\to512\to832]$ with SiLU, lr $10^{-3}$, batch $128$,
$4000$ steps. The counterfactual re-tints the same digit with a random palette color.
Event-accuracy is nearest-neighbour classification of the predicted spectrogram against
the $10$ digit tones ($3000$ evaluation samples; chance $=0.1$).

\subsection{Conditional generation (Table~\ref{tab:gen}).}
Conditional flow matching \citep{lipman2023flow} on the SCM of
Appendix~\ref{app:fvscm}. The velocity field takes
$(\xa,t,E(\xv),s)$ with a video encoder $[d_{\mathrm{v}}\to256\to64]$ and a trunk
$[d_{\mathrm{a}}{+}1{+}64{+}1\to256\to256\to d_{\mathrm{a}}]$. Training samples
$x_0\sim\mathcal{N}(0,I)$, $x_1=\xa$, $t\sim\mathcal{U}[0,1]$,
$x_t=(1-t)x_0+tx_1$, and regresses the target velocity $x_1-x_0$; lr $2\times10^{-3}$,
batch $256$, $4000$ steps. The counterfactual term matches the velocity fields
conditioned on the original and on a nuisance-resampled video. Sampling integrates the
ODE with $20$ Euler steps. Event-accuracy compares the \emph{sampled} audio to the
$K$ event spectra scaled by the sample's $1/s$.

\subsection{Mechanistic probe (Figure~\ref{fig:probe}, left).}
For each trained model we capture the input to its final linear layer---the
audio-producing representation---via a forward hook, on $4000$ samples drawn with
$\nuis$ \emph{decoupled} from $\Ev$ (so that event and nuisance are separable). We then
fit a linear classifier (Adam, lr $10^{-2}$, $800$ steps) on one half of the samples to
predict the event, and, separately, the nuisance, and report accuracy on the held-out
half. Chance is $1/8=0.125$.

\subsection{Multiple nuisances (Figure~\ref{fig:probe}, right).}
$L{=}4$ nuisances, each with $M{=}8$ values and each perfectly confounded with the event
in training; video dimension $48$ with
$W_{\mathrm{v}}\in\mathbb{R}^{(K+LM)\times d_{\mathrm{v}}}$. The intervention resamples
the first $k$ nuisances ($k=0,\dots,4$) while holding the event and the remaining
nuisances fixed. Everything else follows Appendix~\ref{app:fvscm}.

\subsection{Continuous nuisance (Appendix~\ref{app:ablate}).}
The nuisance is a continuous vector $\nuis\in\mathbb{R}^{8}$ drawn as
$\mathrm{emb}(\Ev)+0.25\,\epsilon$ in training, with a fixed random
$\mathrm{emb}\in\mathbb{R}^{K\times8}$ scaled by $1.5$; at test the embedding index is
drawn independently of the event. The counterfactual resamples the continuous nuisance
for a fresh, independently drawn event index.

\subsection{Shared-prior baseline and capacity sweep (Appendix~\ref{app:ablate}).}
The shared-prior model encodes $z_{\mathrm{s}}=E(\xv)\in\mathbb{R}^{8}$ and trains
$\|{\rm aud}(z_{\mathrm{s}},s)-\xa\|^2+\beta\|{\rm vid}(z_{\mathrm{s}})-\xv\|^2$ with
$\beta{=}1$, so that both modalities are aligned to the single shared code. The capacity
sweep varies the hidden width of the \emph{direct} model over
$\{32,64,128,256,512,1024\}$ ($2.7$K--$1.1$M parameters) with all else fixed.

\subsection{Baseline study and label scarcity (Table~\ref{tab:baselines},
Figure~\ref{fig:labelscarce}).}
All five variants share one architecture with an explicit encoder/head split
($[d_{\mathrm{v}}{+}1\to256\to256]$ encoder, linear head to $d_{\mathrm{a}}$) so that
representation-level consistency is well defined; only the loss differs. In the
label-scarce runs a Bernoulli($p$) mask marks which samples carry an audio label: the
supervised terms (including augmentation's counterfactual term) are averaged over
labeled samples only, while the consistency term---which never references
$\xa$---is applied to the whole batch. This is exactly the asymmetry the experiment
measures.

\paragraph{Reproducibility and error bars.} All worlds, model initializations, and
training batches are generated from fixed seeds. We report mean $\pm$ std over independent
training seeds wherever a comparison could plausibly be seed-sensitive: the baseline study
(Table~\ref{tab:baselines}, $5$ seeds), the cue and label-scarcity sweeps
(Figures~\ref{fig:cue}, \ref{fig:labelscarce}, $5$/$3$ seeds), the emergence and DG
baselines, the representation-depth and partially-informative-nuisance ablations
(Appendix~\ref{app:review}, $3$ seeds), and the two headline real-modality tables
(Colored-MNIST Table~\ref{tab:real} and pixel Table~\ref{tab:pixels}, $3$ seeds). We give
per-quantity std where a table has few columns and a summary bound (``std $\le x$'') where a
table is wide, purely for readability; the underlying runs are the same. The
conditional-generation (Table~\ref{tab:gen}), temporal-video (Table~\ref{tab:video}) and
moving-digit (Table~\ref{tab:mm}) tables are likewise reported over $3$ seeds; their
between-method gaps are one-to-three orders of magnitude (e.g.\ OOD accuracy
$0.13$--$0.21$ vs.\ $1.00$) and the measured std is far below those gaps, so the qualitative
conclusion is seed-independent. Finally, the MMAudio fine-tune (Section~\ref{sec:realv2a})
is now run with $3$ independent optimization seeds and evaluated, together with the frozen
baseline, on a common held-out set ($12$ videos $\times 2$ seeds) so that the before/after
comparison is strictly matched; we report mean$\pm$std and, honestly, find that the
intervention stably reduces the targeted grayscale sensitivity but broadly dampens input
sensitivity rather than isolating an appearance shortcut. Other reported numbers are medians
of per-sample errors or means over the evaluation set, as indicated in each caption.

\section{Experiments Moved from the Main Text for Space}
\label{app:extra}

These experiments support the main-text narrative but are placed here for space. They use the same settings as Section~\ref{sec:experiments}.

\subsection{Consistency vs.\ supervised augmentation (full results)}
\label{app:baselines}

If one can synthesize the counterfactual $\xv'=\mathrm{do}(\nuis:=\nuis')$ and one knows
it carries the same audio $\xa$, the obvious alternative is to simply \emph{train on it
with supervision}:
$\mathcal{L}_{\mathrm{aug}}=\|f(\xv,s)-\xa\|^2+\|f(\xv',s)-\xa\|^2$. This baseline is
the sharpest test of our objective, and we report it honestly.
Table~\ref{tab:baselines} compares five variants over five seeds.

\begin{table}[htbp]
\centering
\caption{Is the counterfactual \emph{consistency} objective necessary, or does supervised
counterfactual \emph{augmentation} suffice? Feature-vector SCM, $\rho_{\mathrm{spur}}{=}1$,
mean $\pm$ std over $5$ seeds. With full labels the two are statistically
indistinguishable---the essential ingredient is the intervention, not the loss form.
Representation-level consistency, by contrast, is \emph{not} sufficient.}
\label{tab:baselines}
\begin{tabular}{lcc}
\toprule
Method & IID MSE & OOD-cf MSE \\
\midrule
Direct (no intervention)                & 0.0010 $\pm$ 0.0000 & 0.1827 $\pm$ 0.0029 \\
Supervised CF augmentation              & 0.0011 $\pm$ 0.0000 & 0.0012 $\pm$ 0.0000 \\
CF consistency, prediction-level (ours) & 0.0010 $\pm$ 0.0000 & \textbf{0.0011 $\pm$ 0.0000} \\
CF consistency, representation-level    & 0.0012 $\pm$ 0.0000 & 0.0140 $\pm$ 0.0007 \\
Augmentation $+$ consistency            & 0.0010 $\pm$ 0.0000 & \textbf{0.0011 $\pm$ 0.0000} \\
\bottomrule
\end{tabular}
\end{table}

Two findings follow, and the first is a concession. \textbf{When audio labels are
available for every sample, supervised counterfactual augmentation matches our
consistency objective exactly} ($0.0012$ vs.\ $0.0011$); the two are interchangeable.
We therefore do not claim the consistency \emph{loss} as the contribution. What the
comparison does establish is the paper's actual thesis: every method that
\emph{intervenes on the nuisance} closes the gap ($\sim\!0.001$), while every method that
merely reshapes the representation---the bottleneck, the unsupervised shared/private
factorization, the faithful shared-prior model, and, as the table shows,
representation-level consistency ($0.0140$, $12\times$ worse)---does not. The dividing
line runs between intervention and architecture, not between two losses.

Second, the two are \emph{not} interchangeable once labels are scarce
(Figure~\ref{fig:labelscarce}). The consistency term is \emph{label-free}: it constrains
$f(\xv)$ against $f(\xv')$ without ever needing $\xa$, so it can be applied to unlabeled
video, whereas augmentation can only use counterfactuals of \emph{labeled} samples. As the
labeled fraction falls from $1.0$ to $0.02$, augmentation degrades by $4.6\times$
($0.0012\!\to\!0.0055$) while consistency degrades by $3.3\times$
($0.0011\!\to\!0.0036$), a consistent ${\sim}1.5\times$ advantage at every scarce setting
with non-overlapping error bars. This is the regime real AV data actually lives in---
abundant unlabeled video, few clean paired examples---and it is the one place where the
choice of loss, rather than the presence of the intervention, matters.

\begin{figure}[htbp]
\centering
\includegraphics[width=0.55\linewidth]{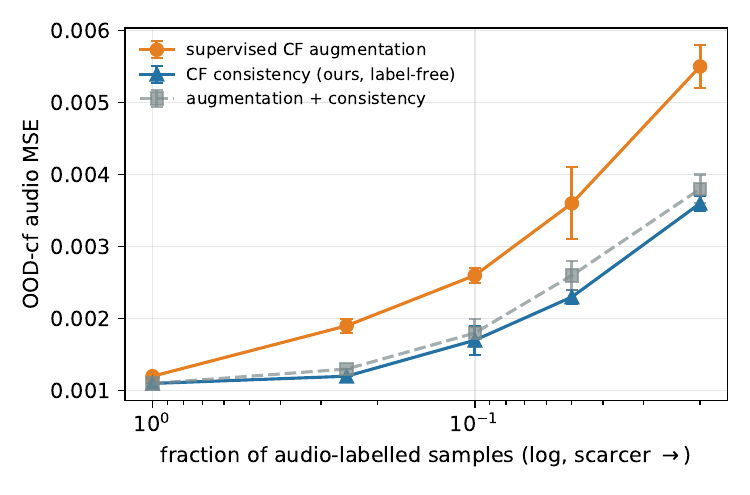}
\caption{Supervised counterfactual augmentation and our consistency objective are
interchangeable when every sample is labeled, but the consistency term needs no audio
label for the counterfactual and therefore keeps improving from unlabeled video as
labels become scarce ($3$ seeds, $\pm$std).}
\label{fig:labelscarce}
\end{figure}

\subsection{When is the shortcut taken? Cue difficulty decides}
\label{sec:cue}

A shortcut is only attractive if it is \emph{easier} than the causal path, so the
phenomenon should not be unconditional---and an honest account must say when it
occurs. We sweep the strength $\gamma$ of the causal event signal in the video while
holding the nuisance a clean one-hot (Figure~\ref{fig:cue}; $5$ seeds, dense sampling
through the transition). The direct model's in-distribution error is flat and near zero
throughout ($0.0010\pm0.0000$), but its counterfactual error falls monotonically and
smoothly from $0.2632\pm0.0006$ at $\gamma{=}0.2$ to $0.0027\pm0.0001$ at $\gamma{=}4$---
a $263\times$ penalty shrinking to $2.7\times$: once the event is easy to read, the model
simply reads it and there is no shortcut to block. The transition
($\gamma\!\in\![0.8,2.0]$) is smooth and stable across seeds (std $\approx1\%$ of the
mean), not a knife-edge. Counterfactual consistency is flat and robust across the whole
sweep and---the practically important point---costs almost nothing where it is
unnecessary ($0.0013$ vs.\ $0.0010$ at $\gamma{=}4$). It is therefore cheap insurance
that can be left on by default.

\begin{figure}[htbp]
\centering
\includegraphics[width=0.55\linewidth]{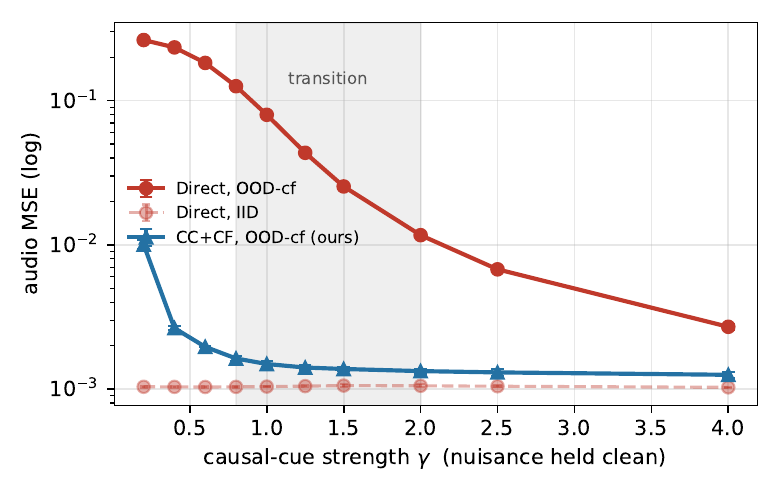}
\caption{The shortcut is conditional. As the causal cue becomes easier to read
($\gamma\uparrow$) the direct model's counterfactual penalty collapses from $263\times$
to $2.7\times$, smoothly and stably across seeds; the counterfactual objective is robust
throughout and nearly free when unnecessary. Shaded: the transition region.
($5$ seeds, $\pm$std; error bars are smaller than the markers.)}
\label{fig:cue}
\end{figure}

\subsection{Real digits with a temporal cause: both regimes, as predicted}

Section~\ref{sec:cue} makes a falsifiable prediction, which we test in the most
realistic setting we can build: a \emph{real} MNIST digit sprite translating in one of
eight directions (the event), tinted with a confounded colour (the nuisance), with the
digit's identity irrelevant and its start position random---so only the motion
identifies the event. Table~\ref{tab:mm} confirms both regimes. When the motion is
easy (a large $0.85$ travel, no noise), the direct model reads it and is
\emph{already robust} ($1.00$ OOD accuracy): there is no shortcut and nothing to fix.
When the motion is made hard (travel $0.30$, frame noise $0.35$), the direct model
switches to the colour and collapses to near chance ($0.21$, chance $=0.125$), while
counterfactual consistency stays at $0.995$. The shortcut is thus neither universal nor
contrived: it appears exactly where the theory says it should, on real pixels with a
genuinely temporal cause.

\begin{table}[htbp]
\centering
\caption{Moving-MNIST (real digit sprites, event $=$ motion direction, nuisance $=$
colour, chance $=0.125$). The shortcut appears only when the causal motion cue is hard,
exactly as the cue-difficulty sweep predicts; the counterfactual objective is robust in
both regimes. Mean$\pm$std over $3$ seeds (OOD columns).}
\label{tab:mm}
\begin{tabular}{llcccc}
\toprule
& & \multicolumn{2}{c}{IID} & \multicolumn{2}{c}{OOD-cf} \\
\cmidrule(lr){3-4}\cmidrule(lr){5-6}
Motion cue & Model & MSE & acc & MSE & acc \\
\midrule
easy (travel $0.85$)            & Direct       & 0.0014 & 1.00 & $0.006{\pm}0.002$ & $1.00{\pm}0.00$ \\
                                & CC+CF (ours) & 0.0009 & 1.00 & $0.001{\pm}0.000$ & $1.00{\pm}0.00$ \\
\midrule
hard (travel $0.30$, noise $0.35$) & Direct    & 0.0015 & 1.00 & $1.180{\pm}0.141$ & $\mathbf{0.209{\pm}0.033}$ \\
                                & CC+CF (ours) & 0.0049 & 1.00 & $0.009{\pm}0.002$ & $\mathbf{0.995{\pm}0.002}$ \\
\bottomrule
\end{tabular}
\end{table}

\subsection{Mechanistic probe: the fix suppresses the nuisance in the audio code}

Why do the models differ? We linearly probe the audio-producing representation
(the input to the final audio layer) for event vs.\ material information, on
decoupled data where the two are separable (Figure~\ref{fig:probe}). In the
direct and bottleneck models the material is fully decodable from the
audio representation (accuracy $1.0$)---the sound is being read off the
nuisance. Counterfactual consistency drives material-decodability down sharply, close to
chance ($0.24$ vs.\ chance $0.125$; not to zero) while keeping event-decodability
at $1.0$: the representation feeding audio is reorganized around the cause, though a
residual trace of the nuisance remains. This is
the mechanism behind the OOD gap. This reorganization holds specifically at the
\emph{output} representation; applying the same consistency constraint at earlier layers
removes the shortcut only partially (Appendix~\ref{app:depth}).

\begin{figure}[htbp]
\centering
\includegraphics[width=0.48\linewidth]{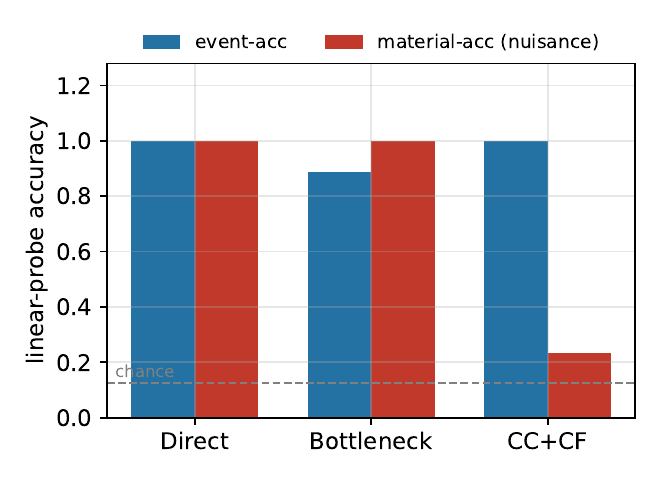}
\hfill
\includegraphics[width=0.48\linewidth]{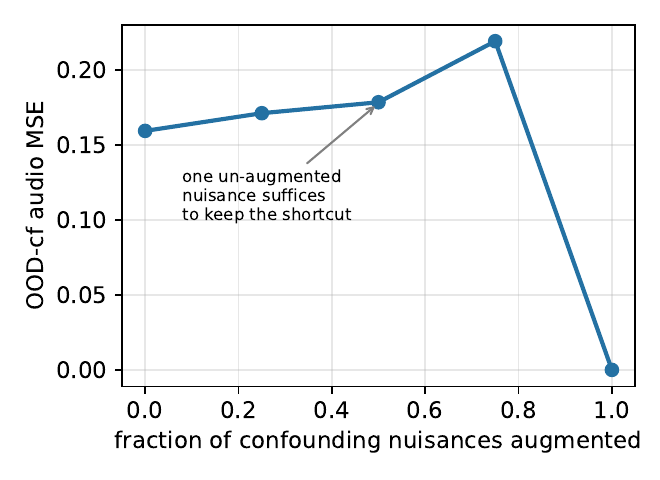}
\caption{\textbf{Left:} linear-probe accuracy of the audio-producing
representation. Direct/bottleneck leak the nuisance (high material-accuracy);
CC+CF suppresses it (to near, not exactly, chance). \textbf{Right:} with multiple perfectly-confounded nuisances,
robustness is all-or-nothing---any single un-augmented nuisance keeps a complete
shortcut, so OOD error only collapses once \emph{every} confounder is
intervened on (experiment of Section~\ref{sec:experiments}, matching
Proposition~\ref{prop:partial}).}
\label{fig:probe}
\end{figure}

\subsection{Multiple nuisances: robustness is all-or-nothing}

Real data confounds through several appearance factors at once, and one may be
able to augment only some of them. With $L{=}4$ nuisances, each perfectly
confounded with the event, we vary the fraction on which the intervention is
applied (Figure~\ref{fig:probe}, right). Because each nuisance individually
determines the event, any single \emph{un}-augmented nuisance is a complete
shortcut: OOD error stays high (indeed slightly worsens, as the model
concentrates on the remaining shortcuts) until \emph{all} confounders are
covered, at which point it collapses to near zero. This is exactly
Proposition~\ref{prop:partial}: partial augmentation buys nothing when nuisances
are individually sufficient. This all-or-nothing behaviour is specific to that regime;
when each nuisance is only partially predictive of the event, partial augmentation
already helps and the response is graded (Appendix~\ref{app:partialinfo2}). Either way it
sharpens the open problem---to be safe one must identify and intervene on the confounding
factors, and knowing which they are is exactly what one lacks in the wild.

\subsection{The mechanism survives real pixels; architecture alone does not}

Table~\ref{tab:pixels} moves to procedurally rendered pixel video: the event is
a shape arrangement (the causal signal), the nuisance is color, color is
perfectly confounded with the shape in training and decoupled at test, and audio
is the event-determined target. On genuine pixels the direct model exploits
color and its OOD-cf error explodes by $100\times$; the bottleneck and the
unsupervised shared/private disentanglement fail just as hard, confirming that
architecture alone does not block the shortcut. Only the counterfactual
intervention---here a random-recolor augmentation, a feasible real-world
counterfactual---remains robust (OOD-cf $\approx$ IID), while fitting the IID
data as well as the direct model.

\begin{table}[htbp]
\centering
\caption{Procedural pixel audio--video ($32{\times}32$ RGB, color = nuisance,
perfectly confounded). Audio MSE, mean over $3$ seeds (std $\le0.001$ IID,
$\le0.055$ OOD). Only the counterfactual intervention (recolor
augmentation) blocks the shortcut; a bottleneck and an unsupervised shared/private
factorization do not.}
\label{tab:pixels}
\begin{tabular}{lccc}
\toprule
Model & IID & OOD-cf & OOD/IID \\
\midrule
Direct                    & 0.010 & 1.042 & 105$\times$ \\
Bottleneck                & 0.011 & 1.155 & 107$\times$ \\
Disentangle (shared/priv) & 0.010 & 1.119 & 109$\times$ \\
CC+CF, recolor (ours)     & 0.010 & \textbf{0.010} & \textbf{1.0$\times$} \\
\bottomrule
\end{tabular}
\end{table}

\subsection{Temporal video: when the causal cue is motion}

All settings so far treat a frame; but in video the causal event is often
\emph{temporal}. We therefore make the event a \emph{motion}: a colored blob
travels in one of $K{=}8$ directions over $T{=}8$ frames, and the audio is
determined by that direction. The blob's start position is randomized, so no
single frame reveals the direction---only the cross-frame displacement does---while
the color is visible in \emph{every} frame and is confounded with the direction in
training (Figure~\ref{fig:vidstrip}). The shortcut is therefore ``read the color from
any frame'' and the causal path requires integrating motion.
Table~\ref{tab:video} shows the direct
model's motion-accuracy falling from $1.00$ to $0.52$ and its error exploding by
$2000\times$ under the counterfactual color shift, whereas counterfactual
consistency---which forces the model onto the temporal cue---is exactly robust
($1.00$ accuracy, error unchanged).

\begin{figure}[htbp]
\centering
\includegraphics[width=0.98\linewidth]{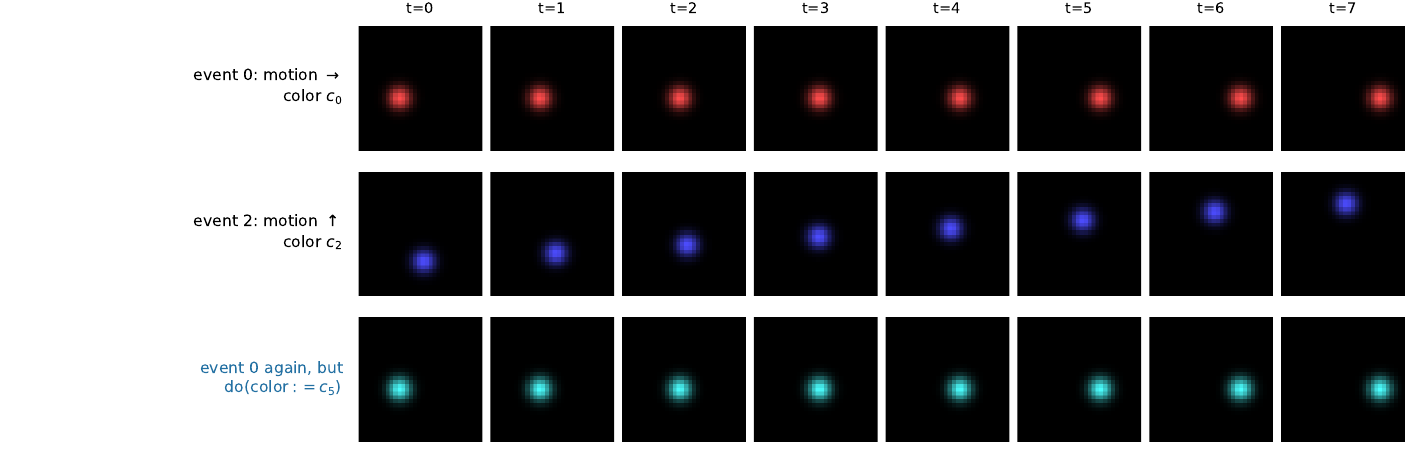}
\caption{The temporal setting. The event is the \emph{motion}, which is identifiable
only across frames because the start position is random; the nuisance colour is
visible in every frame and is therefore the tempting shortcut. The third row is the
counterfactual of the first: identical motion, resampled colour.}
\label{fig:vidstrip}
\end{figure}

\begin{table}[htbp]
\centering
\caption{Temporal video: event $=$ motion direction, nuisance $=$ color
($\rho_{\mathrm{spur}}{=}1$, chance accuracy $=0.125$). Only the cross-frame
motion identifies the event; color is a per-frame shortcut. Counterfactual
consistency forces the model onto the temporal cue. Mean$\pm$std over $3$ seeds.}
\label{tab:video}
\begin{tabular}{lcccc}
\toprule
& \multicolumn{2}{c}{IID} & \multicolumn{2}{c}{OOD-cf} \\
\cmidrule(lr){2-3}\cmidrule(lr){4-5}
Model & MSE & motion-acc & MSE & motion-acc \\
\midrule
Direct        & 0.0003 & 1.00 & $0.664{\pm}0.044$ & $0.517{\pm}0.025$ \\
CC+CF (ours)  & 0.0002 & 1.00 & $\mathbf{0.0002{\pm}0.0001}$ & $\mathbf{1.00{\pm}0.00}$ \\
\bottomrule
\end{tabular}
\end{table}

\subsection{Summary of the experimental findings}
Across a feature-vector SCM, slice-local confounding, procedural pixel video, real
images with spectrogram audio, moving real digits, a conditional generator, and a
mechanistic probe, one line separates what works from what does not, and it is not the
line we might have hoped for. \emph{Intervening} on the nuisance blocks the shortcut---
whether the intervention is delivered as supervised counterfactual augmentation or as a
prediction-level consistency penalty, which are interchangeable when labels are
plentiful (Section~\ref{sec:baselines}). \emph{Reshaping the representation} does not:
the bottleneck, unsupervised shared/private disentanglement, a faithful shared-prior
model, and representation-level consistency all leave the shortcut intact. Two conditions delimit when it matters, and both are empirical, not
assumed: the training confounding must be near-perfect (possibly slice-local), or at least
sufficiently strong---Table~\ref{tab:emergence} shows a clear shortcut already at
$\rho{=}0.99$ under a limited training budget---and the causal cue must be
hard to read relative to the nuisance. Where those conditions fail, the direct model is
already fine---and the objective costs almost nothing there, so it can be left on by
default. Where they hold, it is the difference between synthesizing the right event's
sound and the wrong one.

\paragraph{Qualitative and real-model figures.} Figures~\ref{fig:specfail} (the failure made audible on Colored-MNIST) and~\ref{fig:realv2a} (the real MMAudio generator's non-invariance) accompany the main-text discussion.

\begin{figure}[htbp]
\centering
\includegraphics[width=\linewidth]{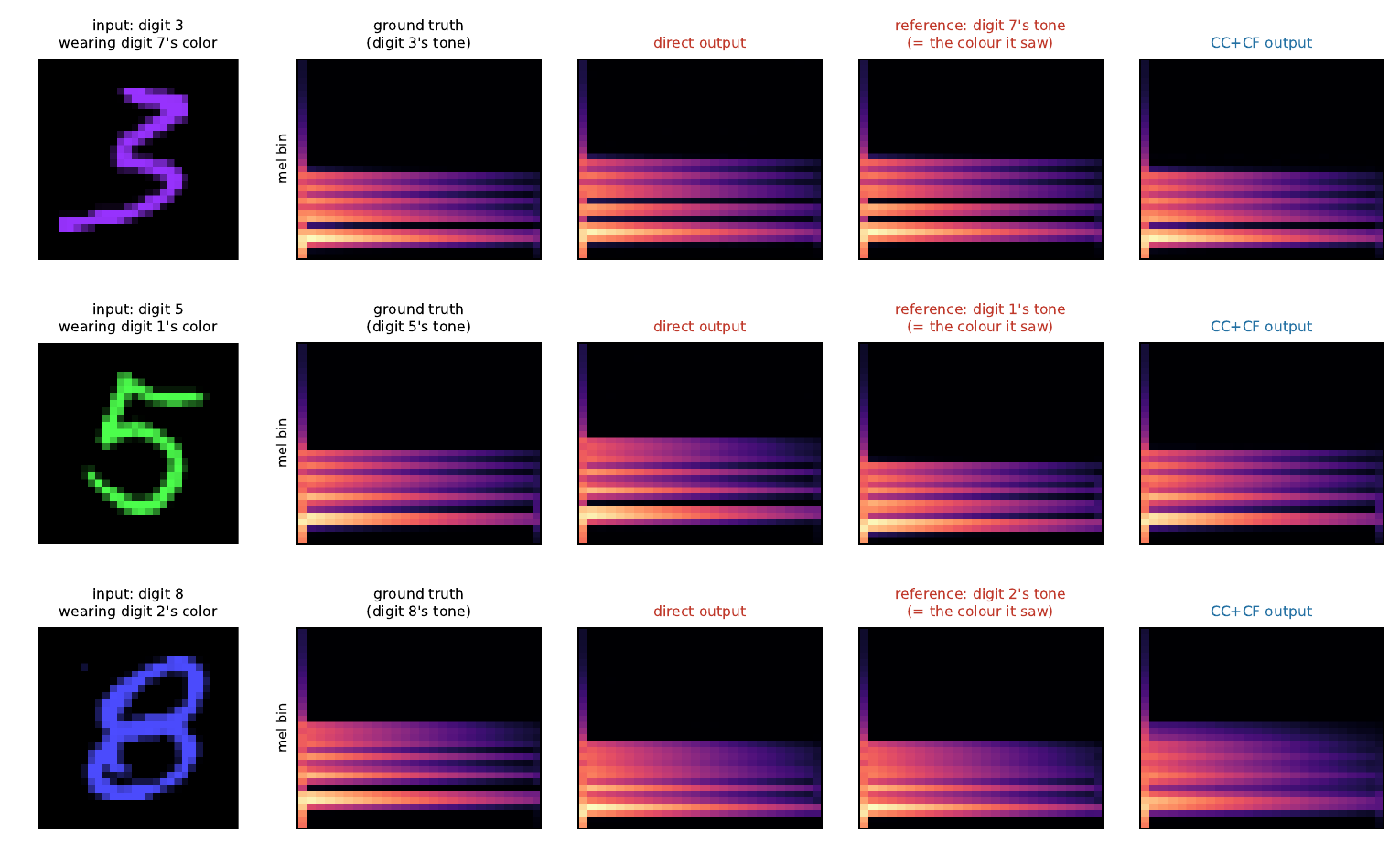}
\caption{\textbf{The failure, made audible.} Each row shows a digit wearing another
digit's colour. The direct model's generated spectrogram (col.~3) matches the tone of
the digit whose \emph{colour} it saw (col.~4)---not the ground truth of the digit it
was actually shown (col.~2). Counterfactual consistency (col.~5) recovers the correct
tone. The shortcut is not a small quantitative degradation: the model synthesizes the
wrong event's sound.}
\label{fig:specfail}
\end{figure}

\begin{figure}[htbp]
\centering
\includegraphics[width=0.66\linewidth]{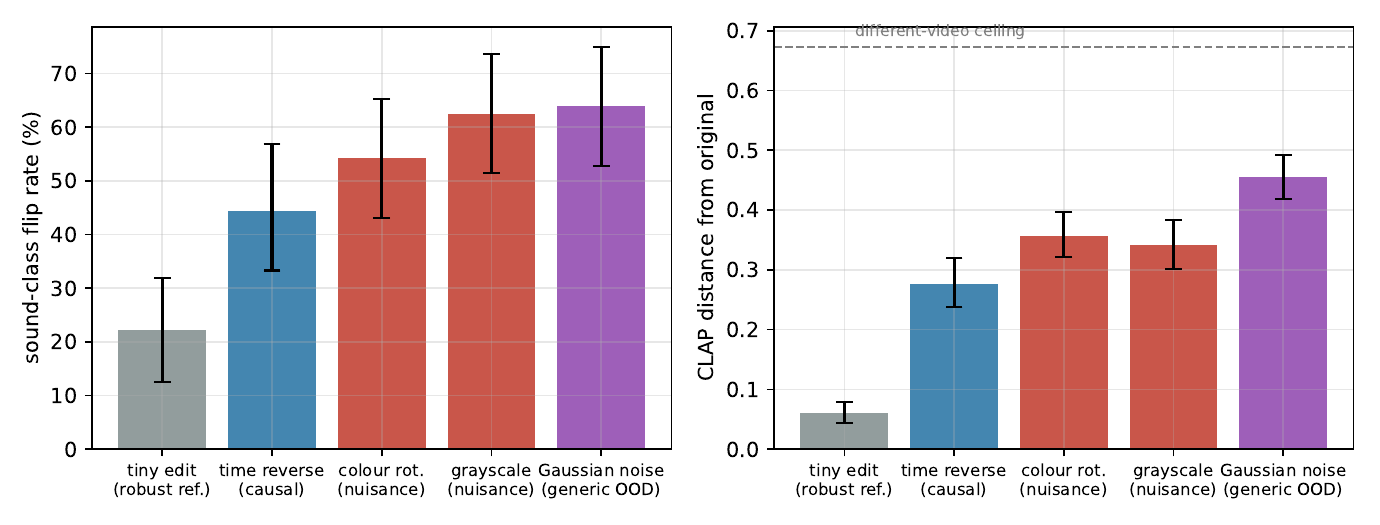}
\caption{A real pretrained video-to-audio generator (MMAudio) is far from invariant to
sound-irrelevant edits, but not specifically to appearance. Sound-class flip rate (left)
and CLAP distance (right) of the generated audio under each input edit ($24$ real videos,
$3$ seeds; $95\%$ bootstrap CIs). Appearance edits (grayscale, colour rotation; red) move
the sound far above the minimal-edit reference---but a generic Gaussian-noise OOD control
(purple), which carries no event information, moves it at least as much. The real model is
broadly input-brittle, and---because the noise control is a larger, unmatched
perturbation---this experiment does \emph{not} establish colour-specific shortcutting.
Isolating a colour-specific shortcut requires the controlled confounds of the synthetic
studies.}
\label{fig:realv2a}
\end{figure}

\section{Additional Robustness Experiments}
\label{app:ablate}

All experiments here use the feature-vector SCM of Section~\ref{sec:experiments}
at $\rho_{\mathrm{spur}}{=}1$.

\paragraph{Faithful shared-prior baseline.}
Table~\ref{tab:shared} gives the full numbers for the shared-prior stand-in
discussed in Section~\ref{sec:experiments}: a single code $z_{\mathrm{s}}$ that
generates audio and reconstructs video (both modalities aligned to it). It
collapses under the counterfactual shift like the direct model, while the
counterfactual objective is $\sim\!76\times$ better.

\begin{table}[h]
\centering
\caption{A faithful shared-prior model (audio generated from, and video
reconstructed from, one shared code) does not block the shortcut.}
\label{tab:shared}
\begin{tabular}{lcc}
\toprule
Model & IID MSE & OOD-cf MSE \\
\midrule
Direct                & 0.0010 & 0.1816 \\
Shared-prior (JavisDiT-style) & 0.0016 & 0.1515 \\
CC+CF (ours)          & 0.0019 & \textbf{0.0020} \\
\bottomrule
\end{tabular}
\end{table}

\paragraph{Weight of the counterfactual term.}
Figure~\ref{fig:ablate} (left) sweeps $\lambda$. Robustness appears as soon as the
term is switched on and holds over two orders of magnitude
($\lambda\in[0.1,10]$: OOD-cf MSE $0.0032\!\to\!0.0020$), while the
in-distribution cost is negligible throughout ($0.0016\!\to\!0.0022$). The method
requires no tuning.

\paragraph{Continuous nuisance.}
The nuisance need not be a discrete palette. With a \emph{continuous} nuisance
vector (an event-dependent embedding plus noise in training, independent at test),
the direct model still takes the shortcut (OOD-cf MSE $0.241$) and the
counterfactual objective, whose intervention resamples the continuous nuisance,
remains robust ($0.0002$).

\paragraph{Capacity does not fix the shortcut.}
Figure~\ref{fig:ablate} (right) scales the direct model from $2.7$K to $1.1$M
parameters ($400\times$). In-distribution error is flat and near zero throughout,
and \emph{the OOD-cf error does not improve at all} ($0.173\!\to\!0.184$). The
shortcut is therefore not an underfitting artifact and cannot be scaled away---it
is a property of what the training distribution identifies
(Proposition~\ref{prop:blind}), not of model capacity.

\begin{figure}[h]
\centering
\includegraphics[width=0.9\linewidth]{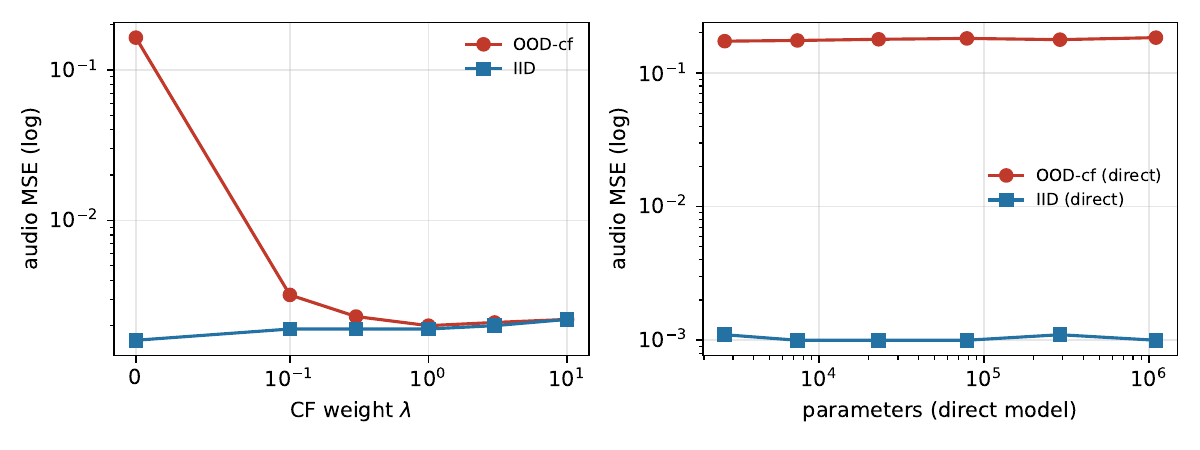}
\caption{\textbf{Left:} the counterfactual term needs no tuning---robustness holds
across two orders of magnitude of $\lambda$ at negligible IID cost.
\textbf{Right:} scaling the direct model $400\times$ leaves the shortcut entirely
intact (OOD-cf flat) while IID error stays near zero.}
\label{fig:ablate}
\end{figure}

\section{When the Shortcut Emerges, and Comparisons to Invariance Baselines}
\label{app:review}

This appendix maps the conditions under which the shortcut becomes catastrophic and
compares the counterfactual intervention against domain-generalization and invariance
methods. All experiments use the feature-vector SCM; error bars, where shown, are over
$3$ seeds.

\subsection{Emergence: confounding strength, training budget, and nuisance richness.}
Proposition~\ref{prop:blind} proves non-identifiability at \emph{perfect} confounding
but does not, by itself, say the shortcut is absent below it. Table~\ref{tab:emergence}
and Figure~\ref{fig:emergence} therefore vary the confounding $\rho_{\mathrm{spur}}$
against the training budget. The failure is \emph{graded}: at $\rho_{\mathrm{spur}}\le0.95$
more training drives the direct model toward the causal solution (the decorrelated
examples eventually dominate), whereas at $\rho_{\mathrm{spur}}{=}1$ more training pushes
it \emph{further} into the shortcut. ``Near-perfect confounding'' is thus where the
failure becomes catastrophic \emph{in this regime}, and we state it that way rather than
as an unconditional law. Richer nuisances make the shortcut easier still: OOD-cf error
rises from $0.003$ at $M{=}2$ nuisance values to $0.13$ at $M{=}8$.

\begin{table}[h]
\centering
\caption{Emergence of the shortcut: direct-model OOD-cf MSE as a function of confounding
$\rho_{\mathrm{spur}}$ and training budget. Below $\rho{=}1$ more training \emph{removes}
the shortcut; at $\rho{=}1$ it entrenches it.}
\label{tab:emergence}
\begin{tabular}{lccc}
\toprule
$\rho_{\mathrm{spur}}$ & 400 steps & 1500 steps & 6000 steps \\
\midrule
0.80 & 0.0033 & 0.0018 & 0.0010 \\
0.90 & 0.0049 & 0.0021 & 0.0011 \\
0.95 & 0.0099 & 0.0028 & 0.0012 \\
0.99 & 0.0519 & 0.0216 & 0.0015 \\
1.00 & 0.1046 & 0.1131 & 0.1320 \\
\bottomrule
\end{tabular}
\end{table}

\begin{figure}[h]
\centering
\includegraphics[width=0.48\linewidth]{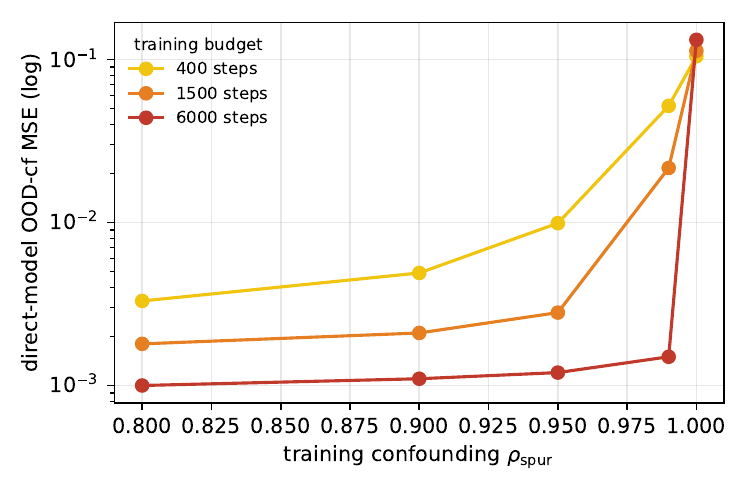}
\hfill
\includegraphics[width=0.48\linewidth]{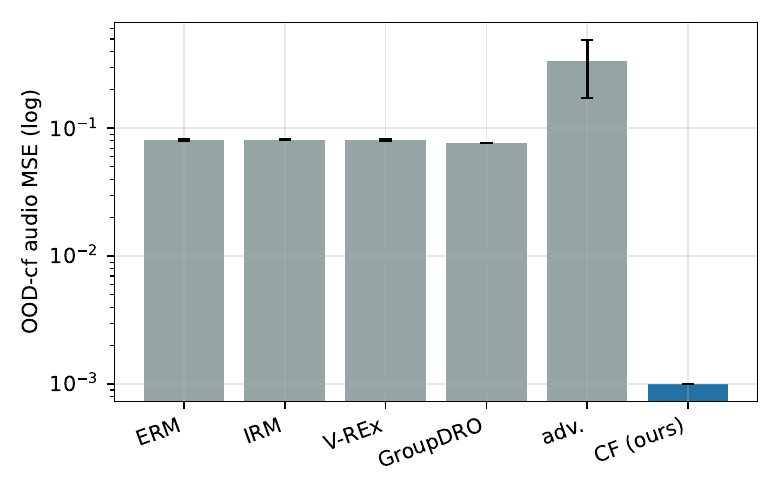}
\caption{\textbf{Left:} the shortcut is graded in confounding and training budget---only
at $\rho_{\mathrm{spur}}{\to}1$ does more training entrench it. \textbf{Right:}
domain-generalization / invariance baselines (given two training environments) versus the
counterfactual intervention (given none); log scale, $\pm$std over $3$ seeds.}
\label{fig:emergence}
\end{figure}

\subsection{Long-tail events and partial coverage.}
Two further slices confirm the mechanism is targeted, not diffuse. With Zipf-distributed
event frequencies, rare (tail) events suffer at least as much as frequent ones (OOD-cf
$0.127$ vs.\ $0.117$): they have the fewest decorrelated examples to learn the causal path
from. And when only \emph{some} event classes are ever seen decorrelated in training, the
shortcut is taken \emph{exactly} on the classes that were never decorrelated (OOD-cf
$0.049$ on confounded-only classes vs.\ $0.0012$ on classes seen decorrelated)---a
class-level version of the slice-locality of Section~\ref{sec:experiments}.

\subsection{Comparison to invariance / domain-generalization baselines.}
A natural question is whether standard invariance methods already solve this. We compare
IRM \citep{arjovsky2019invariant}, V-REx \citep{krueger2021rex},
GroupDRO \citep{sagawa2020group}, and adversarial nuisance removal
\citep{ganin2016domain} against the
counterfactual intervention (Table~\ref{tab:dgbase}, Figure~\ref{fig:emergence} right).
Because these methods require \emph{multiple environments} with varying spurious
correlation, we grant them two training environments with different colour--event maps.
(This multi-environment setup is a harder learning problem than the single confounded
distribution used elsewhere, which is why the absolute IID errors here, $\sim\!0.08$, sit
above the $\sim\!0.001$ of the other tables; the comparison of interest is within this
table, across methods.) The counterfactual intervention is given no environment labels.
With only two
environments whose spurious correlations still overlap, IRM, V-REx, and GroupDRO barely
improve on ERM (OOD-cf $\approx0.08$), and adversarial removal is unstable
($0.33\pm0.16$), while the intervention reaches $0.0010$. The lesson is not that these
methods are bad---given enough diverse environments they can work---but that they trade
one requirement (a known, faithful nuisance intervention) for another (many environments
that decorrelate the nuisance), and the former is available precisely when one can
construct the counterfactual.

\begin{table}[h]
\centering
\caption{Invariance / domain-generalization baselines (two training environments) vs.\ the
counterfactual intervention (no environments). Mean $\pm$ std over $3$ seeds.}
\label{tab:dgbase}
\begin{tabular}{lcc}
\toprule
Method & IID MSE & OOD-cf MSE \\
\midrule
ERM (direct)              & 0.0813 $\pm$ 0.0021 & 0.0810 $\pm$ 0.0012 \\
IRM                       & 0.0820 $\pm$ 0.0019 & 0.0816 $\pm$ 0.0008 \\
V-REx                     & 0.0813 $\pm$ 0.0021 & 0.0811 $\pm$ 0.0012 \\
GroupDRO                  & 0.0769 $\pm$ 0.0020 & 0.0768 $\pm$ 0.0004 \\
Adversarial nuisance removal & 0.3309 $\pm$ 0.1630 & 0.3329 $\pm$ 0.1596 \\
CF intervention (ours)    & 0.0010 $\pm$ 0.0000 & \textbf{0.0010 $\pm$ 0.0000} \\
\bottomrule
\end{tabular}
\end{table}

\paragraph{Baseline implementation details.} For reproducibility we specify the baselines
exactly. \emph{Environments.} All DG baselines receive two training environments that share
the $(\Ev,s)$ structure but differ in the colour--event map: environment~1 uses
$\nuis=\Ev\bmod M$ and environment~2 a fixed derangement $\nuis=\sigma(\Ev)$ with
$\sigma(\Ev)\neq\Ev$, so the spurious correlation is present in both but not identical;
the event marginal is matched across environments. All methods share the same backbone
(the encoder of Appendix~\ref{app:details}) and are trained for the same number of steps
with Adam at the same base learning rate as the direct model. \emph{IRM}
\citep{arjovsky2019invariant}: penalty weight $\lambda_{\mathrm{IRM}}{=}10^{4}$ after a
$500$-step ERM warm-up (swept $\{10^{2},10^{3},10^{4},10^{5}\}$; the reported value is the
best OOD-cf). \emph{V-REx} \citep{krueger2021rex}: variance penalty $\beta{=}10$ on the
per-environment risks, same warm-up (swept $\{1,10,100\}$). \emph{GroupDRO}
\citep{sagawa2020group}: groups are the two environments, group-weight step size
$\eta_q{=}0.01$. \emph{Adversarial nuisance removal} \citep{ganin2016domain}: a
gradient-reversal head predicts $\nuis$ from the audio representation with adversary weight
$\alpha{=}1.0$ (its high variance in Table~\ref{tab:dgbase} reflects the well-known
instability of adversarial training at this scale, not a tuning artefact---it was the best of
$\alpha\in\{0.1,1,10\}$). None of these methods is given the counterfactual pairing; the CF
intervention is given no environment labels.

\subsection{Where to apply the consistency constraint.}
\label{app:depth}
Our claim that reshaping the representation is insufficient should not rest on a single
layer, so we apply the consistency penalty at every depth (Table~\ref{tab:replevel}). The
effect generally improves toward later layers: an input-level penalty does nothing,
penalties at intermediate layers help only partially (and not strictly monotonically---
layer~2 is slightly worse than layer~1), and only an \emph{output}-level (prediction)
penalty fully removes the shortcut. Consistency deep in the network still leaves a usable
nuisance trace in the later layers.

\begin{table}[h]
\centering
\caption{Consistency applied at increasing depth (mean $\pm$ std, $3$ seeds). Only
output-level (prediction) consistency fully blocks the shortcut.}
\label{tab:replevel}
\begin{tabular}{lcc}
\toprule
Consistency at & IID MSE & OOD-cf MSE \\
\midrule
none / input & 0.0010 $\pm$ 0.0000 & 0.1836 $\pm$ 0.0018 \\
layer 1      & 0.0015 $\pm$ 0.0000 & 0.0827 $\pm$ 0.0030 \\
layer 2      & 0.0016 $\pm$ 0.0001 & 0.0979 $\pm$ 0.0021 \\
layer 3      & 0.0011 $\pm$ 0.0000 & 0.0018 $\pm$ 0.0001 \\
output       & 0.0010 $\pm$ 0.0000 & \textbf{0.0010 $\pm$ 0.0000} \\
\bottomrule
\end{tabular}
\end{table}

\subsection{Partially-informative nuisances: the all-or-nothing result is conditional.}
\label{app:partialinfo2}
Proposition~\ref{prop:partial} and its experiment assumed each nuisance
\emph{individually} determines the event, which is what makes partial augmentation useless
there. When instead each of $L{=}3$ nuisances predicts the event only with probability
$q<1$, no single one is a sufficient shortcut, and augmenting a subset already helps: the
response becomes graded (Table~\ref{tab:partialinfo}). The all-or-nothing behaviour is
therefore a property of the individually-sufficient regime, not a universal law.

\begin{table}[h]
\centering
\caption{Partially-informative nuisances: OOD-cf MSE as a function of how many of $L{=}3$
nuisances are augmented, for different per-nuisance predictiveness $q$. At $q{=}1$ the row
is all-or-nothing; for $q<1$ partial augmentation already helps.}
\label{tab:partialinfo}
\begin{tabular}{lcccc}
\toprule
$q$ & augment 0/3 & 1/3 & 2/3 & 3/3 \\
\midrule
1.00 & 0.1597 & 0.1710 & 0.2005 & 0.0001 \\
0.80 & 0.0040 & 0.0005 & 0.0002 & 0.0001 \\
0.60 & 0.0006 & 0.0002 & 0.0001 & 0.0001 \\
\bottomrule
\end{tabular}
\end{table}

\section*{AI Use Statement}
Parts of this paper's text, code and analysis were produced with the assistance of generative
AI. The authors take full responsibility for the final content of this paper, including any
text produced with the assistance of generative AI.

\end{document}

%% file: math_commands.tex
\usepackage{amsmath,amsfonts,bm}

\def\eqref#1{equation~\ref{#1}}

\def\1{\bm{1}}

\DeclareMathAlphabet{\mathsfit}{\encodingdefault}{\sfdefault}{m}{sl}
\SetMathAlphabet{\mathsfit}{bold}{\encodingdefault}{\sfdefault}{bx}{n}

